\documentclass[journal]{IEEEtran}
\usepackage{cite}
\usepackage{amsmath,amssymb,amsfonts}
\usepackage{graphicx}
\usepackage{textcomp}
\usepackage{caption}

\usepackage{bm}

\newcommand{\blue}[1]{#1}

\newcommand{\smax}{\sigma_\mathrm{max}}
\newcommand{\smin}{\sigma_\mathrm{min}}

\newcommand{\bx}{{\bm x}}
\newcommand{\by}{{\bm y}}
\newcommand{\bw}{{\bm w}}
\newcommand{\bg}{{\bm g}}
\newcommand{\bb}{{\bm b}}
\newcommand{\bA}{\mathbf{A}}
\newcommand{\bC}{\mathbf{C}}
\newcommand{\bH}{\mathbf{H}}
\newcommand{\bF}{\mathbf{F}}
\newcommand{\bG}{\mathbf{G}}
\newcommand{\bI}{\mathbf{I}}

\newcommand{\bX}{{\mathbf X}}
\newcommand{\bY}{{\mathbf Y}}
\newcommand{\bW}{{\mathbf W}}
\newcommand{\bB}{{\mathbf B}}
\newcommand{\bD}{{\mathbf D}}
\newcommand{\TV}{{\mathrm{TV}}}
\newcommand{\Z}{\mathbb{Z}}

\newcommand{\cA}{\mathcal{A}}

\newcommand{\cF}{\mathcal{F}}

\newcommand{\cN}{\mathcal{N}}

\newcommand{\SR}{SR3}

\DeclareMathOperator{\prox}{prox}
\DeclareMathOperator{\proj}{proj}

\DeclareMathOperator*{\argmin}{argmin}
\usepackage{overpic}

\usepackage[noend]{algpseudocode}
\usepackage{algorithm}
\usepackage{algorithmicx}
\newcommand*\Let[2]{\State #1 $\gets$ #2}
\algrenewcommand\algorithmicrequire{\textbf{Input:}}
\algrenewcommand\algorithmicensure{\textbf{Input:}}

\usepackage{graphicx}

\usepackage{amsthm}
\usepackage{bm}
\usepackage{subcaption}

\newcommand{\mbf}{\mathbf}
\newcommand{\bc}{{\bm c}}
\newcommand{\bd}{{\bm d}}
\newcommand{\bh}{{\bm h}}
\newcommand{\bv}{{\bm v}}
\newcommand{\R}{\mathbb{R}}
\newcommand{\ip}[1]{\left\langle #1 \right\rangle}

\DeclareMathOperator{\cond}{cond}
\newcommand{\vectorize}{\mathrm{vec}}

\newtheorem{theorem}{Theorem}
\newtheorem{definition}{Definition}
\newtheorem{lemma}{Lemma}
\newtheorem{corollary}{Corollary}

\DeclareMathOperator{\lip}{Lip}

\DeclareMathOperator{\Diag}{Diag}

\DeclareMathOperator{\sign}{sign}

\usepackage[noend]{algpseudocode}
\usepackage{algorithmicx}

\algrenewcommand\algorithmicrequire{\textbf{Input:}}
\algrenewcommand\algorithmicensure{\textbf{Output:}}

\usepackage{stackengine,scalerel}
\def\elsq{{\mathbin{\ThisStyle{\stackinset{c}{-.1\LMpt}{c}{}{%
  $\SavedStyle 2$}{$\SavedStyle\bigcirc$}}}}}

\title{A Unified Framework for Sparse Relaxed Regularized Regression: SR3}

\author{{Peng Zheng}$^\dagger$\thanks{$^\dagger$Department of Applied Mathematics, University of Washington, Seattle, WA
(zhengp@uw.edu). A. Aravkin was supported in part by the Washington 
  Research Foundation Data Science Professorship.},
{Travis Askham}$^\dagger$\thanks{$^\dagger$Department of Applied Mathematics, University of Washington, Seattle, WA
(askham@uw.edu).  J. N. Kutz and T.~Askham
  acknowledge support from the Air Force Office of Scientific
  Research (FA9550-17-1-0329).}, 
{Steven L. Brunton}$^*$\thanks{$^*$Department of Mechanical Engineering, University of Washington, Seattle, WA
(sbrunton@uw.edu). S. L. Brunton acknowledges
  support from the Army Research Office through the
  Young Investigator Program (W911NF-17-1-0422).},
{J. Nathan Kutz}$^\dagger$\thanks{$^\dagger$Department of Applied Mathematics, University of Washington, Seattle, WA
(kutz@uw.edu)},
and 
{Aleksandr Y. Aravkin}$^\dagger$\thanks{$^\dagger$Department of Applied Mathematics, University of Washington, Seattle, WA
(saravkin@uw.edu)}
}

\begin{document}

\maketitle

\begin{abstract}
Regularized regression problems are ubiquitous in statistical modeling, signal processing, 
and machine learning. 
Sparse regression in particular has been instrumental in scientific model discovery, including compressed sensing applications, variable selection, and high-dimensional analysis. 
We propose a broad framework for sparse relaxed regularized regression, called \SR. 
The key idea is to solve a {\it relaxation} of the regularized problem, which has three advantages over the state-of-the-art: (1) solutions of the relaxed problem are superior with respect to errors, false positives, and conditioning, (2) relaxation allows extremely fast algorithms for both convex and nonconvex formulations, and (3) the methods apply to composite regularizers such as total variation (TV) and its nonconvex variants. 
We demonstrate the advantages of \SR~ (computational efficiency, higher accuracy, faster convergence rates, greater flexibility) across a range of regularized regression problems with synthetic and real data, including applications in compressed sensing, LASSO, matrix completion, TV regularization, and group sparsity. 
To promote reproducible research, we also provide a companion \textsc{Matlab} package that implements these examples. 
\end{abstract}

\section{Introduction}
Regression is a cornerstone of data science.
In the age of big data, optimization algorithms are largely focused on regression problems in machine learning and AI. 
%
As data volumes increase, algorithms must be fast, scalable, and robust to low-fidelity measurements (missing data, outliers, etc).
Regularization, which includes priors and constraints, is essential for the recovery of interpretable solutions in high-dimensional and ill-posed settings.
Sparsity-promoting regression is one such fundamental technique, that enforces solution parsimony by balancing model error with complexity.
Despite tremendous methodological progress over the last 80 years, many difficulties remain, including (i) restrictive theoretical conditions for practical performance,
(ii) the lack of fast solvers for large scale and ill-conditioned problems,  
(iii) practical difficulties with nonconvex implementations,
and (iv) high-fidelity requirements on data.
To overcome these difficulties, we propose a broadly applicable method, {\em sparse relaxed regularized regression} (\SR), based on a relaxation reformulation of {\it any} regularized regression problem.
We demonstrate that \SR~is fast, scalable,  robust to noisy and missing data, and flexible enough to apply broadly to regularized regression problems, ranging from the ubiquitous
LASSO and compressed sensing (CS), to composite regularizers such as the total variation (TV) regularization,
and even to nonconvex regularizers, including $\ell_0$ and rank.
\SR~improves on the
state-of-the-art on all of these applications, both in terms of computational speed and performance.
Moreover, \SR~is flexible and simple to implement.
A companion open source package implements a range of examples using \SR.

The origins of regression extend back more than two centuries to the pioneering mathematical contributions of Legendre~\cite{legendre1805nouvelles} and Gauss~\cite{gauss1809theoria,gauss1821theory}, who were interested in determining the orbits of celestial bodies. 
%
%
%
%
%
%
The invention of the digital electronic computer in the mid 20th century greatly increased interest in regression methods, as computations became faster and larger problems from a variety of fields became tractable.
It was recognized early on that many regression problems are ill-posed in nature, either being under-determined, resulting in an infinite set of candidate solutions, or otherwise sensitive to perturbations in the observations, often due to some redundancy in the set of possible models.
%
%
Andrey Tikhonov~\cite{tihonov1943ob} was the first to systematically study the use of regularizers to achieve stable and unique numerical solutions of such ill-posed problems.
%
The regularized linear least squares problem is given by 
\begin{equation}
\min_{\bx}~~\frac{1}{2}\|\bA \bx-\bb\|^2 + \lambda  R(\bC \bx) \, ,
\label{eq:generic}
\end{equation}
\blue{
where $\bx \in \mathbb{R}^d$ is the unknown signal, $\bA \in \mathbb{R}^{m\times d}$ is the linear data-generating mechanism for the observations $\bb \in \mathbb{R}^m$, 
$\bC\in \mathbb{R}^{n\times d}$ is a linear map,
$R(\cdot)$ is any regularizer, and $\lambda$ parametrizes the strength of the regularization.}
Tikhonov proposed a simple $\ell_2$ penalty, 
i.e. $R({\bx})=\| {\bx}\|^2 = \sum x_i^2$, which eventually led to the formal introduction of the {\em ridge} regression strategy by Hoerl and Kennard 30 years later~\cite{hoerl1976ridge}.
Other important regularizers include the $\ell_0$ penalty,
$R({\bx})=\|\bx\|_0$, and the sparsity-promoting convex $\ell_1$ relaxation $R({\bx})=\|\bx\|_1$, introduced by Chen and Donoho in 1994~\cite{shaobing1994basis} as {\it basis pursuit}, and by Tibshirani in 1996~\cite{tibshirani1996regression} as the {\em least absolute shrinkage and selection operator} (LASSO). More generally,  the $\ell_1$ norm was introduced much earlier: as a {penalty} in 1969~\cite{pietrzykowski1969exact}, with specialized algorithms in 1973~\cite{conn1973constrained}, 
and as a robust loss in geophysics in 1973~\cite{claerbout1973robust}.
In modern optimization, nonsmooth regularizers are widely used across a diverse set of applications, including in the training of neural network architectures~\cite{Goodfellow-et-al-2016}.
Figure~\ref{fig:lasso_main}(a) illustrates the classic sparse regression iteration procedure for LASSO. Given the 1-norm of the solution, i.e. $\|\hat \bx\|_1 = \tau$, the solution can be found by
`inflating' the level set of the data misfit until it intersects the ball $\mathbb{B}_1 \leq \tau$.  The geometry of the level sets influences both the robustness of the procedure with 
respect to noise, and the convergence rate of iterative algorithms used to find $\hat \bx$.

\noindent
\blue{
{\bf Contributions.}
In this paper, we propose a broad framework for sparse relaxed regularized regression, called \SR. 
The key idea of \SR~is to solve a regularized problem that has three advantages over the state-of-the-art: (1) solutions are superior with respect to errors, false positives, and conditioning, (2) relaxation allows extremely fast algorithms for both convex and nonconvex formulations, and (3) the methods apply to composite regularizers. 
Rigorous theoretical results supporting these claims are presented in Section~\ref{sec:method}.
We demonstrate the advantages of \SR~ (computational efficiency, higher accuracy, faster convergence rates, greater flexibility) across a range of regularized regression problems with synthetic and real data, including applications in compressed sensing, LASSO, matrix completion, TV regularization, and group sparsity using a range of 
test problems in Section~\ref{sec:results}. 
}



\section{\SR~Method}
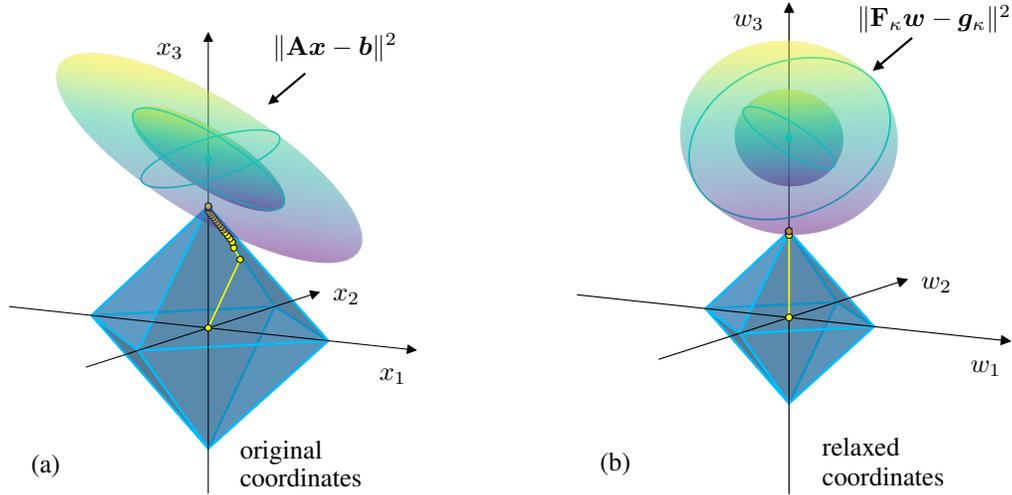
\begin{figure*}
\centering
\begin{overpic}[width=0.3\textwidth]{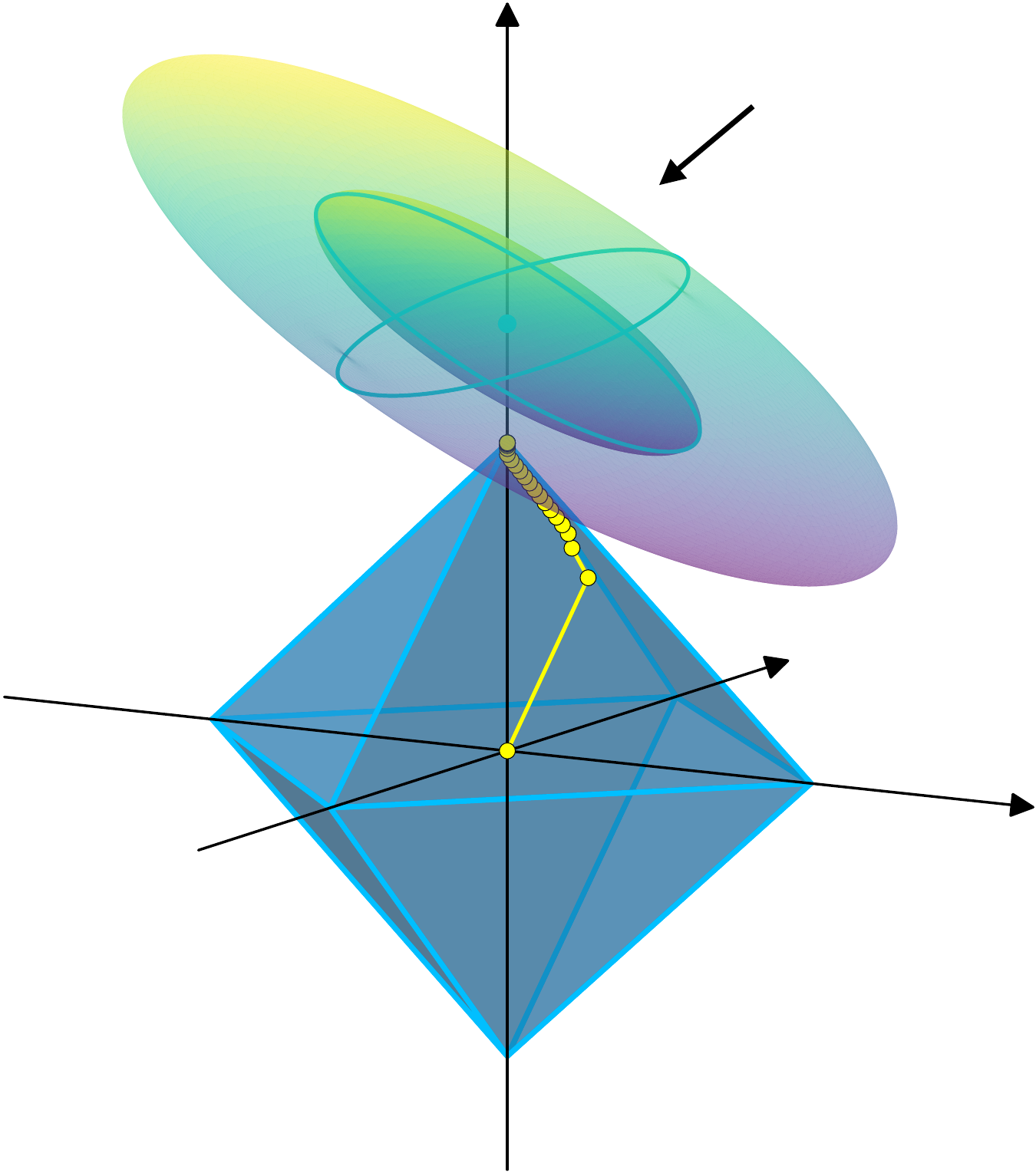} 
\put(80,25){$x_1$}
\put(70,42){$x_2$}
\put(32,95){$x_3$}
\put(50,8){original}
\put(50,2){coordinates}
\put(5,5){(a)}
\put(57,95){$\|\bA \bx - \bb\|^2$}
\end{overpic}\qquad\qquad\qquad
\begin{overpic}[width=0.32\textwidth]{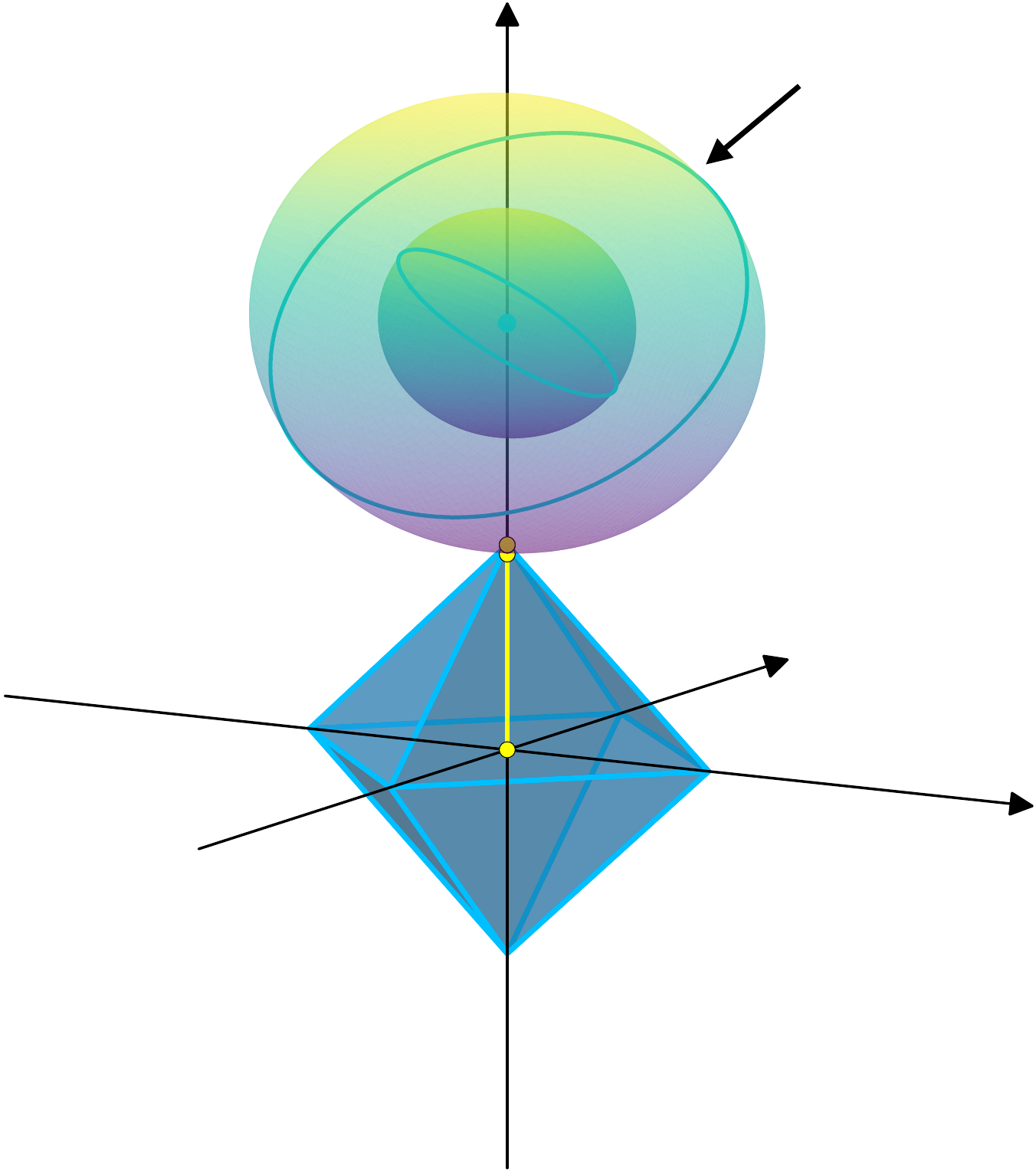} 
\put(80,25){$w_1$}
\put(70,42){$w_2$}
\put(32,95){$w_3$}
\put(50,8){relaxed}
\put(50,2){coordinates}
\put(5,5){(b)}
\put(57,95){$\|\bF_\kappa \bw - \bg_\kappa\|^2$}
\end{overpic}
\caption{\label{fig:lasso_main}
(a) Level sets (green ellipses) of the quadratic part of LASSO (\ref{eq:generic}) and corresponding path of prox-gradient to the solution (40 iterations) in $\bx$-coordinates. (b) Level sets (green spheres) of the quadratic part of the \SR~value function~(\ref{eq:value}) and corresponding \SR~solution path (2 iterations) in relaxed coordinates $\bw$.  Blue octahedra show the $\ell_1$ ball in each set of coordinates. 
\SR~decreases the singular values of $\bF_\kappa$ relative to those of $\bA$ with a weaker effect on the small ones,
`squashing' the level sets into approximate spheres, accelerating convergence, and improving performance.}
\end{figure*}

Our goal is to improve the robustness, computational efficiency, and accuracy of sparse and nonsmooth formulations.
We {\it relax}~\eqref{eq:generic} using an auxiliary variable \blue{$\bw\in\mathbb{R}^{n}$} that is forced to be close to $\mathbf{C}\bx$.
Relaxation was recently shown to be an efficient technique for dealing with the class of nonconvex-composite problems~\cite{zheng2018fast}.
The general \SR~formulation modifies~\eqref{eq:generic} to the following
\begin{equation}
\label{eq:generalxw}
\min_{\bx,\bw} \frac{1}{2}\|\bA\bx-\bb\|^2 + \lambda R(\bw) + \frac{\kappa}{2} \|\mathbf{C}\bx-\bw\|^2,
\end{equation}
where $\kappa$ is a relaxation parameter that controls the gap between $\bC\bx$ and $\bw$. 
Importantly, $\kappa$ controls both the strength of the improvements to the geometry/regularity of the relaxed problem relative to the original and the fidelity of the relaxed problem to the original.
To recover a relaxed version of LASSO, for example, we take $R(\cdot)= \|\cdot\|_1$ and $\mathbf{C} = \mathbf{I}$. 
The \SR~formulation allows non-convex $\ell_p$ ``norms'' with $p<1$, as well as smoothly clipped absolute deviation~(SCAD)~\cite{fan2001variable},
and easily handles linear composite regularizers. 
\blue{Two widely used examples that rely on compositions are compressed sensing formulations that use tight frames~\cite{donoho2006compressed}, 
and total variation (TV) regularization in image denoising~\cite{rudin1992nonlinear}.}

In the convex setting, the formulation~(\ref{eq:generalxw}) fits into a class of problems studied by Bauschke, Combettes, and Noll~\cite{bauschke2006joint}, 
who credit the natural alternating minimization algorithm to Acker and Prestel in 1980~\cite{acker1980convergence},
and the original alternating projections method to Cheney and Goldstein in 1959~\cite{cheney1959proximity}
and Von Neumann in 1950~\cite[Theorem 13.7]{von1950functional}.
The main novelty of the \SR~approach is in using~\eqref{eq:generalxw} to extract 
information from the $\bw$ variable. We also allow nonconvex regularizers $R(\cdot)$, 
using the structure of~\eqref{eq:generalxw} to simplify the analysis.

The success of \SR~stems from two key ideas.
First, {sparsity} and {accuracy} requirements are split between ${\bw}$ and ${\bx}$ in the formulation~\eqref{eq:generalxw}, 
relieving the pressure these competing goals put on $\bx$ in~\eqref{eq:generic}. 
Second, we can partially 
minimize~\eqref{eq:generalxw} in $\bx$ to obtain a function in $\bw$ alone, with nearly spherical level sets, 
in contrast to the elongated elliptical level sets of $\|\mathbf{A}\bx - \bb\|^2$. 
In $\bw$ coordinates, it is much easier to find the correct support. 
 Figure~\ref{fig:lasso_main}(b) illustrates this advantage of \SR~on the LASSO problem.

\subsection{\SR~and Value Function Optimization}
\label{sec:theory}
Associated with~\eqref{eq:generalxw} is a {\it value function} formulation that allows us to precisely characterize the relaxed framework.  Minimizing~\eqref{eq:generalxw} in $\bx$, we obtain the {value function}
\begin{equation}
\label{eq:value}
v(\bw) := \min_{\bx} \frac{1}{2}\|\bA\bx-\bb\|^2 + \frac{\kappa}{2} \|\mathbf{C}\bx-\bw\|^2. 
\end{equation}
\blue{
We assume that $\bH_\kappa = \bA^\top \bA + \kappa\bC^\top \bC $ is invertible. 
 Under this assumption, $\bx(\bw) = \bH_\kappa^{-1}\left(\bA^\top\bb + \kappa\bC^\top\bw\right)$ is unique.
We now define   
\begin{equation} \label{eq:Fdef}
\begin{aligned}
\bF_\kappa &= \begin{bmatrix}
\kappa \bA  \bH_\kappa^{-1}\bC^\top\\
\sqrt{\kappa}(\bI - \kappa \bC \bH_\kappa^{-1} \bC^\top)
\end{bmatrix},   && \bF_\kappa \in \mathbb{R}^{(m + n) \times n} \\
\bG_\kappa &= \begin{bmatrix}
\bI - \bA \bH_\kappa^{-1} \bA^\top\\
\sqrt{\kappa}\bC \bH_\kappa^{-1} \bA^\top
\end{bmatrix}, && \bG_\kappa \in \mathbb{R}^{(m + n) \times m} \\
\bg_\kappa &= \bG_\kappa \bb, && \bg_\kappa \in \mathbb{R}^{m+n}
\end{aligned}
\end{equation}
}
which gives a closed form for \eqref{eq:value}: 
\[
v(\bw) = \frac{1}{2}\|\bF_\kappa \bw - \bg_\kappa\|^2.
\]
Problem~\eqref{eq:generalxw}  then reduces to    
\begin{equation}
\label{eq:valueExp}
\min_{\bw} \frac{1}{2}\|\bF_\kappa \bw - \bg_\kappa\|^2 + \lambda R(\bw) \; .
\end{equation}
The ellipsoid in Fig.~\ref{fig:lasso_main}(a) shows the level sets of  $\|\bA\bx-\bb\|^2$, 
while the spheroid in Fig.~\ref{fig:lasso_main}(b) shows
the level sets of $\|\bF_\kappa \bw - \bg_\kappa\|^2$. 
Partial minimization improves the conditioning of the problem, as seen in Figure~\ref{fig:lasso_main},
and can be characterized by a simple theorem. 

\blue{Denote by $\sigma_i(\cdot)$ the function that returns the
  $i$-th largest singular value of the argument, with $\sigma_{\mathrm{max}}(\bA)$
  denoting the largest singular value $\sigma_1(\bA)$, and
   $\sigma_{\mathrm{min}}(\bA)$ denoting the smallest (reduced) singular
  value $\sigma_{\min(m,d)}(\bA)$.
  Let $\cond(\bA):= \sigma_{\mathrm{max}}(\bA)/\sigma_{\mathrm{min}}(\bA)$
  denote the condition number of $\bA$.}
\blue{The following result relates singular values of $\bF_\kappa$ to those of $\bA$ 
and $\bC$. Stronger results apply to the special cases $\bC= \bI$, 
which covers the Lasso, and $\bC^\top \bC = \bI$, 
which covers compressed sensing formulations with  
tight frames ($\bC = \mathbf{\Phi}^\top$ with $\mathbf{\Phi}\mathbf{\Phi}^\top = \bI$)~\cite{chen2001atomic,donoho2006compressed,elad2007analysis}.  
}
%
%
\begin{theorem}
\label{thm:sol}
When $\lambda = 0$, \eqref{eq:valueExp} and \eqref{eq:generic} share the same solution set.
\blue{
We also have the following relations: 
\begin{align}
\label{eq:FtF}
\mbf F_\kappa^\top \mbf F_\kappa &= \kappa \mbf I - \kappa^2 \bC \mbf H_\kappa^{-1} \bC^\top \\
\label{eq:svFtF}
\sigma_i(\mbf F_\kappa^\top \mbf F_\kappa) &= \kappa - \kappa^2\sigma_{n-i+1}(\bC \mbf H_\kappa^{-1} \bC^\top).
\end{align}
In addition, $\mbf 0 \preceq \mbf F_\kappa^\top \mbf F_\kappa \preceq \kappa \mbf I$ always, and
when $n \ge d$ and $\bC$ has full rank (i.e. $\bC^\top\bC$ is invertible), we have  
\[
\smin(\mbf F_\kappa^\top \mbf F_\kappa) \ge \frac{\smin(\mbf A^\top \mbf A)/\smax(\bC^\top \bC)}{1 + \smin(\mbf A^\top \mbf A)/(\kappa\smax(\bC^\top \bC))}.
\]
\blue{When $\bC = \mbf I$, we have}
\begin{align}
\label{eq:FtFCI}
\mbf F_\kappa^\top \mbf F_\kappa &= \mbf A^\top(\mbf I + \mbf A \mbf A^\top/\kappa)^{-1}\mbf A\\
\label{eq:svFtFCI}
\sigma_i(\mbf F_\kappa^\top \mbf F_\kappa) &=\frac{\sigma_i(\mbf A^\top \mbf A)}{1 + \sigma_i(\mbf A^\top \mbf A)/\kappa} \; ,
\end{align}
so that the condition numbers of $\bF_\kappa$
and $\bA$ are related by
\begin{equation} \label{eq:condf}
  \cond(\bF_\kappa) = \cond(\bA) \sqrt{ \frac{\kappa+\sigma_{\mathrm{min}}(\bA)^2}
    {\kappa+\sigma_{\mathrm{max}}(\bA)^2}} \; .
\end{equation}
}
\end{theorem}
%
Theorem~\ref{thm:sol} lets us interpret
\eqref{eq:valueExp} as a re-weighted version of the original problem \eqref{eq:generic}.
\blue{In the general case, the properties of $\bF$ depend on the interplay between $\bA$ and 
$\bC$. The re-weighted linear map $\bF_\kappa$ has superior 
properties to $\bA$ in special cases. Theorem~\ref{thm:sol} gives strong results for $\bC =\bI$, 
and we can derive analogous results when
  $\bC$ has orthogonal columns and full rank.}

\begin{corollary} \label{cor:tight_frame}
  \blue{Suppose that $\bC \in \R^{n\times d}$ with $n \geq d$
    and $\bC^\top \bC = \mbf I_d$.}
  \blue{Then, 
    \begin{equation}
      \sigma_i(\mbf F_\kappa) = \left \{ \begin{array}{lr}
        \sqrt{\kappa} \ \frac{\sigma_{i-(n-d)}(\mbf A)}{\sqrt{\kappa + \sigma_{i-(n-d)}(\mbf A)^2}} & i > n-d \\
        \sqrt{\kappa} & i \leq n-d
      \end{array} \right. .
    \end{equation}
    For $n > d$, this implies
  \begin{equation}
    \label{eq:condffc}
    \cond(\bF_\kappa) = \cond(\bA) \sqrt{ \frac{\kappa+\sigma_{\mathrm{min}}(\bA)^2}
      {\sigma_{\mathrm{max}}(\bA)^2}} \; .
  \end{equation}
  When $n=d$, this implies
  \begin{equation} \label{eq:condf}
    \cond(\bF_\kappa) = \cond(\bA) \sqrt{ \frac{\kappa+\sigma_{\mathrm{min}}(\bA)^2}
      {\kappa+\sigma_{\mathrm{max}}(\bA)^2}} \; .
  \end{equation}
  }
  \begin{proof}
    \blue{Let $\bar{\bC} = \begin{bmatrix} \bC &
        \bC^\perp\end{bmatrix}$ where the columns of
        $\bC^\perp$ form an orthonormal
        basis for the orthogonal complement of the
        range of $\bC$. Then, by Theorem~\ref{thm:sol},
        \begin{equation}
          \label{eq:cffc}
          \bar{\bC}^\top \bF_\kappa^\top \bF_\kappa \bar{\bC} = \begin{bmatrix} \mbf A^\top(\mbf I + \mbf A \mbf A^\top/\kappa)^{-1}\mbf A& \\ & \kappa \mbf I_{n-d} \end{bmatrix} \; .
        \end{equation}
The result follows from the second part of Theorem~\ref{thm:sol} .}
  \end{proof}
\end{corollary}

\blue{When $\bC$ is a square orthogonal matrix, }partial minimization of~\eqref{eq:value} 
shrinks the singular values of $\bF_\kappa$ relative to $\bA$,
with less shrinkage for  smaller singular values,
\blue{which gives a smaller condition number
  as seen in Figure~\ref{fig:lasso_main} for $\bC = \bI$.}
As a result, iterative methods
for \eqref{eq:valueExp}  converge much faster than
the same methods applied to \eqref{eq:generic}, 
especially for ill-conditioned $\bA$.
The geometry of the level sets of~\eqref{eq:valueExp} 
also encourages the discovery of sparse solutions;
see \blue{the path-to-solution for each formulation in}
Figure~\ref{fig:lasso_main}.
\blue{The amount of improvement depends on the size
  of $\kappa$, with smaller values of $\kappa$ giving better
  conditioned problems. For instance, consider setting
  $\kappa = (\sigma_{\mathrm{max}}(\bA)^2-\sigma_{\mathrm{min}}(\bA)^2)/\mu^2$
  for some $\mu > 1$.
  Then, by Corollary~\ref{cor:tight_frame},
  $\cond (\bF_\kappa) \leq 1 + \cond (\bA)/\mu$.}

%


%

%
%
%

  \subsection{Algorithms for the \SR~Problem}
  Problem \eqref{eq:valueExp} can be solved using a
 variety of algorithms, including the prox-gradient
 method detailed in Algorithm~\ref{alg:pg_x_gen}. 
 In the convex case, Algorithm~\ref{alg:pg_x_gen} is equivalent to the alternating method of~\cite{bauschke2006joint}. 
 The $\bw$ update is given by
\begin{equation}
\label{w-update}
\hat \bw^{k+1} = \prox_{\frac{\lambda}{\kappa} R}\left(\bw^k - \frac{1}{\kappa}\bF_\kappa^\top(\bF_\kappa \bw^k - \bg_\kappa)\right) \; ,
\end{equation}
where $\prox_{\frac{\lambda}{\kappa} R}$ is the {\it proximity operator} (prox) for $R$ (see e.g.~\cite{combettes2011proximal}) evaluated at $\bC \bx$.
\begin{algorithm}[t]
\caption{\SR~for~\eqref{eq:generalxw}}
\label{alg:pg_x_gen}
\begin{algorithmic}[1]
\State {\bfseries Input:} $\bw^0$
\State {\bfseries Initialize:} $k=0$, $\eta \leq \frac{1}{\kappa}$ 
\While{not converged}
\Let{$k$}{$k+1$}
\Let{$\bw^{k}$}{$\prox_{\eta\lambda R}(\bw^{k-1} - \eta \bF_\kappa^\top(\bF_\kappa \bw^{k-1} - \bg_\kappa))$}
\EndWhile
\State {\bfseries Output:} $\bw^k$
\end{algorithmic}
\end{algorithm}
The prox in Algorithm~\ref{alg:pg_x_gen} is
easy to evaluate for many important convex and nonconvex functions,
often taking the form of a separable atomic operator, i.e. the prox requires a simple computation for each individual entry of the
input vector. 
For example, $\prox_{\lambda \|\cdot\|_1}$ is the {\it soft-thresholding} (ST) operator:
\begin{equation}
\label{eq:ST}
\prox_{\lambda \|\cdot\|_1}(\bx)_i = \mbox{sign}(x_i) \max(|x_i|-\lambda, 0).
\end{equation}

Algorithm~\ref{alg:pg_x_gen} is the proximal gradient algorithm applied to~\eqref{eq:valueExp}.
It is useful to contrast it with the proximal gradient algorithm for the original problem~\eqref{eq:generic}, detailed in Algorithm~\ref{alg:pg}. 
\begin{algorithm}[t]
\caption{Prox-gradient~for~\eqref{eq:generic}}
\label{alg:pg}
\begin{algorithmic}[1]
\State {\bfseries Input:} $\bx^0$
\State {\bfseries Initialize:} $k=0, \eta \leq \frac{1}{\sigma_{\max}(\bA)^2}$
\While{not converged}
\Let{$k$}{$k+1$}
\Let{$\bx^{k}$}{$\prox_{\eta\lambda R(\bC\cdot)}(\bx^{k-1} - \eta\bA^\top(\bA\bx^{k-1} - b))$}
\EndWhile
\State {\bfseries Output:} $\bx^k$
\end{algorithmic}
\end{algorithm}
First, Algorithm~\ref{alg:pg} may be difficult to implement when $\bC \neq \bI$,
as the prox operator may no longer be separable or atomic. 
An iterative algorithm is required to evaluate 
\begin{equation}
\label{eq:STC}
\prox_{\lambda \|\bC\cdot\|_1}(\bx) = \arg\min_{\by} \frac{1}{2\lambda}\|\bx - \by\|^2  + \|\bC\by\|_1.
\end{equation}
\blue{
In contrast, Algorithm 1 always solves~\eqref{eq:valueExp}, which is regularized by $R(\bw)$ rather than a composition, 
 with $\bC$ affecting $\bF_\kappa$ and $\bg_\kappa$, see~\eqref{eq:Fdef}. This simple observation has important 
 consequences, since the prox-gradient method converges for a wide class of problems, including non-convex 
 regularizers~\cite{attouch2010proximal}. For regularized least squares problems specifically, we  
 derive a self-contained convergence theorem with a sublinear convergence rate. 
}

\blue{
\begin{theorem}[Proximal Gradient Descent for Regularized Least Squares]
\label{thm:general_lr}
Consider the linear regression objective,
\[
\min_{\bx}~p(\bx) := \frac{1}{2}\|\bA \bx - \bb\|^2 + \lambda R(\bx) \;,
\]
where $p$ is bounded below, so that 
\[
-\infty < p^* = \inf_{\bx} p(\bx),
\]
and $R$ may be nonsmooth and nonconvex.
With step $t=1/\sigma_\mathrm{max}(\bA)^2$, the
iterates generated by Algorithm~\ref{alg:pg}
satisfy 
\[
\bv_{k+1} := (\|\bA\|_2^2 \mathbf I - \bA^\top\bA)(\bx_k - \bx_{k+1}) \in \partial p(\bx_{k+1}),
\]
i.e. $\bv_{k+1}$ is an element of the subdifferential of $p(\bx)$ at the point $\bx_{k+1}$\footnote{For nonconvex problems, the subdifferential must be carefully defined; see
the preliminaries in the Appendix.}, and
\[
\min_{k = 0, \dots N} \|\bv_{k+1}\|^2 \leq  \frac{1}{N}\sum_{k=0}^{N-1}\|\bv_{k+1}\|^2 \le \frac{\|\bA\|_2^2}{N}(p(\bx_0) - p^*) \;.
\]
Therefore Algorithm~\ref{alg:pg} converges at a sublinear rate to a stationary point of $p$.
\end{theorem}
Theorem~\ref{thm:general_lr} always applies to the \SR~approach, which uses value function~\eqref{eq:valueExp}. 
When $\bC = \bI$, we can also compare the convergence rate of Algorithm~\ref{alg:pg_x_gen} for~\eqref{eq:valueExp}
to the rate for Algorithm~\ref{alg:pg} for~\eqref{alg:pg}.
}
In particular, the rates of Algorithm~\ref{alg:pg_x_gen} are independent of $\bA$ when $\bA$ does not have full rank, 
and depend only weakly on $\bA$ when $\bA$ has  full rank,  as detailed in Theorem~\ref{col:cvx}.   
\begin{theorem}
\label{col:cvx}
Suppose that $\bC = \bI$.
Let
  $\bx^*$ and $\bw^*$ denote the minimum values of
  $p_x(\bx) := \frac{1}{2}\|\bA \bx - \bb \|^2 + R(\bx)$ and
  $p_w(\bw) := \frac{1}{2}\|\bF_\kappa\bw-\bg_\kappa \|^2 + R(\bw)$, respectively.
  Let $\bx_k$ denote the iterates of Algorithm~\ref{alg:pg} applied to $p_x$, and 
$\bw_k$ denote the iterates of Algorithm~\ref{alg:pg_x_gen} applied to $p_w$, 
 with step sizes  $\eta_x = \frac{1}{\sigma_\mathrm{max}(\bA)^2}$ and $\eta_w = \frac{1}{\sigma_\mathrm{max}(\bF_\kappa)^2}$. 
The iterates always satisfy 
  \[
  \begin{aligned} \bv^x_{k+1} & = (\|\bA\|_2^2 \mathbf I - \bA^\top\bA)(\bx_k - \bx_{k+1}) \in \partial p_x(\bx_{k+1}) \\
\bv^w_{k+1} &= (\kappa \mathbf I - \bF^\top\bF)(\bw_k - \bw_{k+1}) \in \partial p_w(\bw_{k+1}).
\end{aligned}
\]
\blue{
For general $R$ and any $\bA$ we have the following rates: 
  \begin{align*}
    \frac{1}{N}\sum_{k=0}^{N-1}\|\bv_{k+1}^x\|^2 &\le \frac{\|\bA\|_2^2}{N}(p_x(\bx_0) - p_x^*)\\
    \frac{1}{N}\sum_{k=0}^{N-1}\|\bv_{k+1}^w\|^2 &\le \frac{\kappa}{N}(p_w(\bx_0) - p_w^*).
  \end{align*}
}
For convex $R$ and any $\bA$ we also have
\begin{align*}
\frac{p_x(\bx) - p_x(\bx^*)}{\|\bx^0 - \bx^*\|^2} &\le \frac{\sigma_\mathrm{max}(\bA)^2}{2(k+1)}\\
\frac{p_w(\bw) - p_w(\bw^*)}{\|\bw^0 - \bw^*\|^2} &\le \frac{\sigma_\mathrm{max}(\bF_\kappa)^2 }{2(k+1)} \\
&\blue{ \le \frac{\frac{\sigma_{\mathrm{max}}(\bA)^2}{1 + \sigma_{\mathrm{max}}(\bA)^2/\kappa}}{2(k+1)} \le \frac{\kappa}{2(k+1)}.}
\end{align*}

For convex $R$ and $\bA$ with full rank, we also have 
\begin{align*}
\frac{\|\bx^k - \bx^*\|^2}{\|\bx^0 - \bx^*\|^2} &\le \left(1 - \frac{\sigma_\mathrm{min}(\bA)^2}{\sigma_\mathrm{max}(\bA)^2}\right)^k\\
\frac{\|\bw^k - \bw^*\|^2}{\|\bw^0 - \bw^*\|^2} &\le \left(1 - \frac{\sigma_\mathrm{min}(\bA)^2}{\sigma_\mathrm{max}(\bA)^2}\frac{\sigma_\mathrm{max}(\bA)^2 + \kappa}{\sigma_\mathrm{min}(\bA)^2 + \kappa}\right)^k
\end{align*}
\end{theorem}

\blue{
When $\bC^\top \bC = \bI$,  algorithm~\ref{alg:pg} may not be implementable. However, \SR~is implementable, 
with rates equal to those for the $\bC=\bI$ case when $n=d$ and with rates as in the following corollary
when $n > d$.
\begin{corollary} When $\bC^\top \bC = \bI$ and $n > d$, let $\bw^*$ denote the minimum value of
  $p_w(\bw) := \frac{1}{2}\|\bF_\kappa\bw-\bg_\kappa \|^2 + R(\bw)$, and let 
$\bw_k$ denote the iterates of Algorithm~\ref{alg:pg_x_gen} applied to $p_w$, 
 with step size  $\eta_w = \frac{1}{\kappa}$.
The iterates always satisfy 
\[
\bv^w_{k+1} = (\kappa \mathbf I - \bF^\top\bF)(\bw_k - \bw_{k+1}) \in \partial p_w(\bw_{k+1}).
\]
For general $R$ and any $\bA$ we have the following rates: 
  \begin{align*}
    \frac{1}{N}\sum_{k=0}^{N-1}\|\bv_{k+1}^w\|^2 &\le \frac{\kappa}{N}(p_w(\bx_0) - p_w^*).
  \end{align*}
For convex $R$ and any $\bA$ we also have
\begin{align*}
\frac{p_w(\bw) - p_w(\bw^*)}{\|\bw^0 - \bw^*\|^2} &\le \frac{\kappa}{2(k+1)} 
\end{align*}
For convex $R$ and $\bA$ with full rank, we also have 
\begin{align*}
\frac{\|\bw^k - \bw^*\|^2}{\|\bw^0 - \bw^*\|^2} &\le \left(1 -\frac{\sigma_{\min}(\bA^\top\bA)}{\kappa + \sigma_{\min}(\bA^\top\bA)}\right)^k
\end{align*}
\label{cor:ctcrates}
\end{corollary} 
}

Algorithm~\ref{alg:pg_x_gen} can be used with both convex and nonconvex regularizers, 
as long as the prox operator of the regularizer is available.
\blue{
 A growing list of proximal 
operators is reviewed by~\cite{combettes2011proximal}. 
Notable nonconvex prox operators in the literature include 
(1) indicator of set of rank $r$ matrices, (2) spectral functions (with proximable outer functions)~\cite{drusvyatskiy2015variational,lewis1999nonsmooth}, 
(3) indicators of unions of convex sets (project onto each and then choose the closest point), 
(4) MCP penalty~\cite{zhang2010nearly}, 
(5) firm-thresholding penalty~\cite{gao1997waveshrink}, and 
(6) indicator functions of finite sets (e.g., $x \in \{-1, 0, 1\}^d$).
Several nonconvex prox operators specifically used in sparse regression are detailed in the next section. 
}

\subsection{\blue{Nonconvex Regularizers and Constraints}}

\begin{figure}[!t]
	\centering
	\begin{subfigure}[t]{0.23\textwidth}
		\begin{center}
		\includegraphics{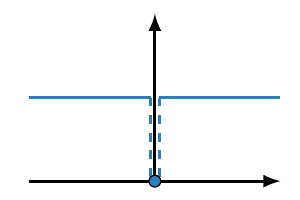}
		\end{center}
		\caption{$\ell_0$ norm.}
	\end{subfigure}
	\begin{subfigure}[t]{0.23\textwidth}
		\begin{center}
		\includegraphics{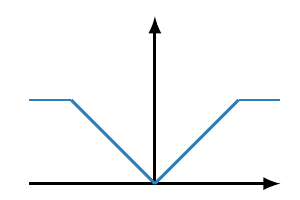}
		\end{center}
		\caption{Clipped absolute deviation.}
	\end{subfigure}\\
	\begin{subfigure}[t]{0.23\textwidth}
		\begin{center}
		\includegraphics{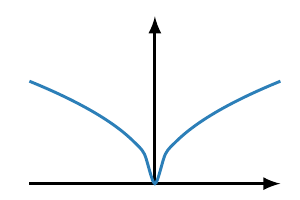}
		\end{center}
		\caption{$\ell_p$ norm ($p  = \frac{1}{2}$).}
	\end{subfigure}
	\begin{subfigure}[t]{0.23\textwidth}
		\begin{center}
		\includegraphics{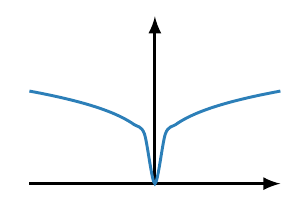}
		\end{center}
		\caption{$\ell_p$ norm ($p = \frac{1}{4}$).}
	\end{subfigure}
	\caption{Nonconvex sparsity promoting regularizers.}
	\label{Fig:geo_regularizers}
\end{figure}

\subsubsection{Nonconvex Regularizers: $\ell_0$.}
The 1-norm is often used as a convex alternative to $\ell_0$, 
defined by $\| {\bx} \|_0 = | \{ i: x_i \ne 0 \} |$, \blue{see panel (a) of Figure~\ref{Fig:geo_regularizers}}.
The nonconvex $\ell_0$ has a simple prox --- hard thresholding (HT)~\cite{blumensath2009iterative}, \blue{see Table~\ref{table:prox}}.
The \SR~formulation with the $\ell_0$ regularizer uses HT instead of the ST
operator \eqref{eq:ST} in line 5 of Algorithm~\ref{alg:pg_x_gen}.

%

\begin{table*}[h]
\centering
\blue{
\begin{tabular}{c | c | c | c }
$R(\bm x)$ & $r(x)$ & $\prox_{\alpha r}(z)$ &  Solution \\ \hline\hline
$\|\bm x\|_1$ & $|x|$ & $\begin{cases} \sign(z)(|z| - \alpha), &|z| > \alpha\\0, &|z| \le \alpha \end{cases}$ & Analytic\\ \hline
$\|\bm x\|_0$ & $\begin{cases} 1, & x \ne 0\\ 0, & x = 0 \end{cases}$ & $\begin{cases} 0, & |z| \le \sqrt{2\alpha}\\ z, & |z| > \sqrt{2\alpha} \end{cases}$ & Analytic \\ \hline
$\|\bm x\|_p^p$ $(p < 1)$ & $|x|^p$ & see Appendix & Coordinate-wise Newton \\ \hline
$\mbox{CAD}(\bm x; \rho)$ & $\begin{cases} |x|, & |x| \le \rho \\ \rho, & |x| > \rho \end{cases}$ & $\begin{cases} z, & |z| > \rho \\ \sign(z)(|z| - \alpha), & \alpha<|z|\le\rho  
\\ 0, &|z| \le \alpha \end{cases}$ & Analytic\\ \hline\hline
\end{tabular}
}
\caption{\label{table:prox}\blue{Proximal operators of sparsity-promoting regularizers.}}
\end{table*}

\subsubsection{\blue{Nonconvex Regularizers: $\ell_p^p$ for $p \in (0,1)$}}
\label{sec:lp}
\blue{
The $\ell_p^p$ regularizer for $p\in(0,1)$ is often used for sparsity promotion, see e.g.~\cite{lai2013improved} and the references within.
\blue{Two members of this family are shown in panels (c) and (d) of Figure~\ref{Fig:geo_regularizers}.}
The $\ell_p^p$ prox subproblem is given by 
\begin{equation}
\label{eq:1D_lp}
\min_x~f_{\alpha,p}(x;z) := \frac{1}{2\alpha}(x - z)^2 + |x|^p
\end{equation}
This problem is studied in detail by~\cite{chen2016computing}. 
Closed form solutions are available for special cases $p \in \left\{\frac{1}{2}, \frac{2}{3}\right\}$;
but a provably convergent Newton method is available for all $p$.
Using a simple method {\it for each coordinate}, we can globally solve 
the nonconvex problem~\eqref{eq:1D_lp}~\cite[Proposition 8]{chen2016computing}. 
Our implementation is summarized in the Appendix. 
The $\ell_{1/2}$ regularizer is particularly important for CS, and 
is known to do better than either $\ell_0$ or $\ell_1$.
}

\begin{figure*}[!t]
\begin{center}
\includegraphics[width=.9\textwidth]{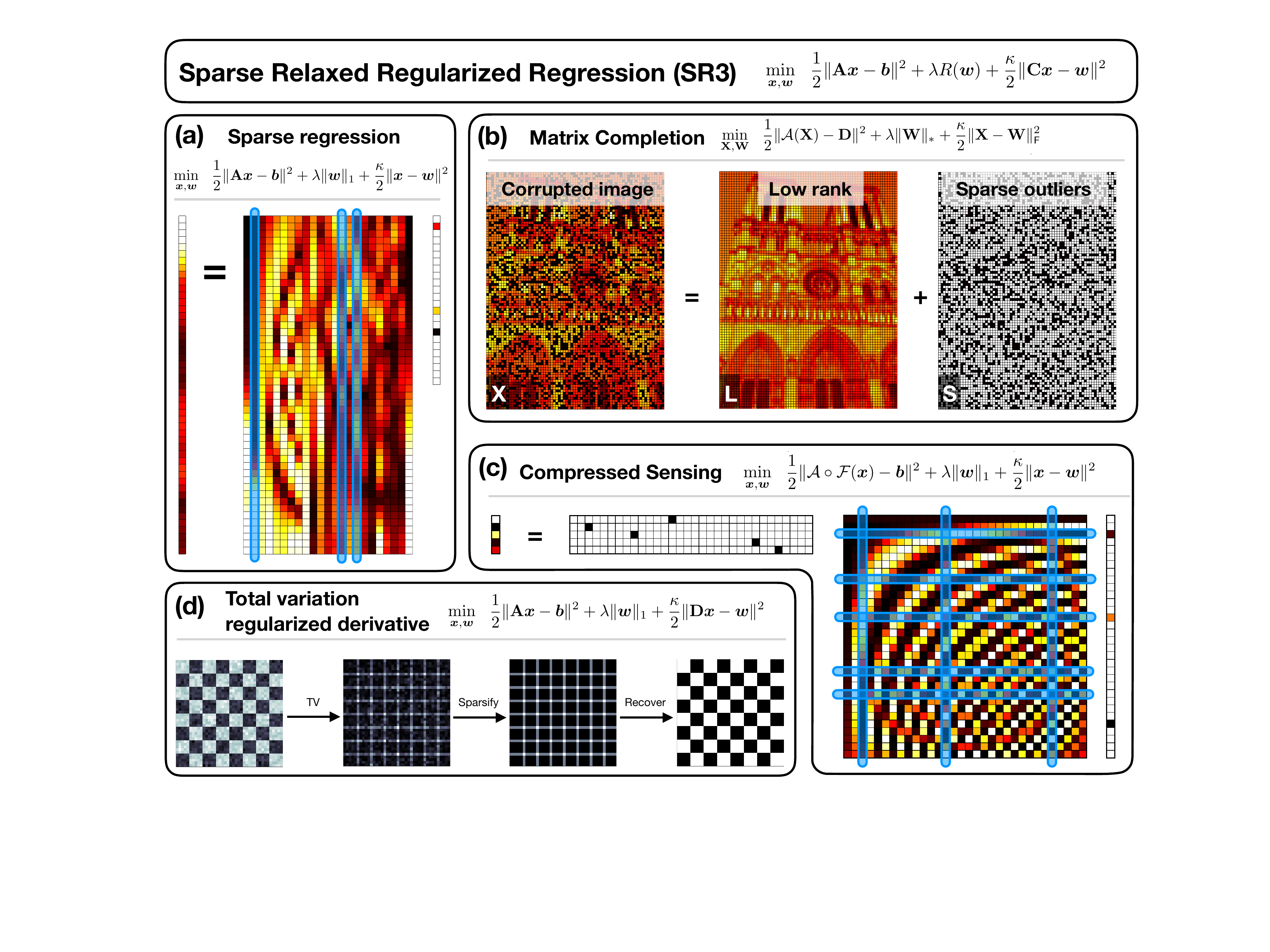}
\end{center}
\caption{Common optimization applications where the \SR~method improves performance.  For each method, the specific implementation of our general architecture (\ref{eq:generalxw}) is given.  }\label{Fig:OverviewBig}
\end{figure*}

\subsubsection{\blue{Nonconvex Regularizers: (S)CAD}}

\blue{The (Smoothly) Clipped Absolute Deviation (SCAD)~\cite{fan2001variable} is a sparsity promoting regularizer used to reduce 
bias in the computed solutions. A simple un-smoothed version (CAD) appears in panel (b) of Figure~\ref{Fig:geo_regularizers}, and the 
analytic prox is given in Table~\ref{table:prox}. This regularizer, when combined with \SR, obtains the best results in the CS experiments in Section~\ref{sec:results}.}

\subsubsection{Composite Regularization: Total Variation (TV).}
TV regularization can be written as 
$\mathrm{TV}(\bx) = R(\bC\bx) = \|\bC\bx\|_1$,  
%
%
with $\bC$ a (sparse) difference matrix (see~\eqref{eq:tvnorm}). 
The \SR~formulation is solved by Algorithm~\ref{alg:pg_x_gen}, 
a prox-gradient (primal) method. 
In contrast, most TV algorithms use primal-dual methods 
because of the composition~$\|\bC\bx\|_1$~\cite{chan2011augmented}.

\subsubsection{Constraints as Infinite-Valued Regularizers.}  The term $R(\cdot)$ does not need 
to be finite valued. In particular, for any set $C$ that has a projection, we can take $R(\cdot)$
to be the indicator function of $C$,  given by 
\[
R_C(\bx) = \begin{cases} 
0 & \bx \in C \\
\infty & \bx \not\in C.
\end{cases},
\] 
so that $\prox_{R}(\bx) = \proj_C(\bx)$. Simple examples of such regularizers include 
convex non-negativity constraints ($\bx\geq 0$) and nonconvex spherical constraints ($\|\bx\|_2 = r$).

\subsection{Optimality of \SR~Solutions}

We now consider the relationship between 
the optimal solution $\hat\bw$ to problem~\eqref{eq:valueExp}, 
and the original problem~\eqref{eq:generic}.

\begin{theorem}[Optimal Ratio]
\label{thm:ratio}
Assume $\bC = \bI$, and let $\lambda_1$ for~\eqref{eq:generic} and $\lambda_2$ for~\eqref{eq:valueExp}  be related by
the ratio $\tau = \lambda_2/\lambda_1$,  and let $\hat \bw^k$ be the optimal solution for~\eqref{eq:valueExp} 
with parameter $\lambda_2$. If  $\lambda_2$ is set to be $\tau\lambda_1$ where
\[
\hat \tau = \argmin_{\tau > 0}~\|\tau \mbf I - \kappa \bH_\kappa^{-1}\|_2 = \frac{\kappa}{2}(\sigma_\mathrm{max}(\bH_\kappa^{-1}) + \sigma_\mathrm{min}(\bH_\kappa^{-1})) \; ,
\]
then have that the distance to optimality of $\hat\bw^1$ for~\eqref{eq:generic} 
is bounded above by 
\[
\frac{\sigma_\mathrm{max}(\bA)^2 - \sigma_\mathrm{min}(\bA)^2}{\sigma_\mathrm{max}(\bA)^2 + \sigma_\mathrm{min}(\bA)^2 + 2\kappa}\|\bA^\top \bA \hat \bw - \bA^\top \bb\|.
\]
\end{theorem}
Theorem~\ref{thm:ratio} gives a way to choose $\lambda_2$ given $\lambda_1$
so that $\hat\bw$ is as close as possible to the stationary point of \eqref{eq:generic}, 
and characterizes the distance of $\hat \bw$ to optimality of the original problem. 
The proof is given in the Appendix. 

Theorem~\ref{thm:ratio} shows that as $\kappa$ increases, the solution $\hat \bw$ 
moves closer to being optimal for the original problem~\eqref{eq:generic}. 
On the other hand, Theorem~\ref{col:cvx}  suggests that lower $\kappa$ values regularize the problem, 
making it easier to solve. 
In practice,  we find that $\hat \bw$ is useful and informative in a range of applications with 
moderate values of $\kappa$, see Section~\ref{sec:results}.

\section{Results}
\label{sec:results}

The formulation~\eqref{eq:generic} covers many standard problems, 
including variable selection (LASSO), compressed sensing,
TV-based image de-noising, and matrix completion, shown in Fig.~\ref{Fig:OverviewBig}.
In this section, we demonstrate the general flexibility of the \SR~formulation and its advantages over other state-of-the-art techniques. 
In particular, \SR~is faster than competing algorithms, and $\bw$ is far more useful in identifying the support of sparse signals, particularly 
when data are noisy and $\mathbf{A}$ is ill-conditioned.  
%

\subsection{\SR~vs. LASSO and Compressed Sensing}
%
Using Eqs.~(\ref{eq:generic}) and (\ref{eq:generalxw}),
the LASSO and associated \SR~problems are 
\begin{eqnarray}
\label{eq:lasso}
&& \min_{\bx}~~\frac{1}{2}\|\bA\bx - \bb\|^2 + \lambda\|\bx\|_1 \\
\label{eq:lasso_sr3}
&& \min_{\bx, \bw}~~\frac{1}{2}\|\bA\bx - \bb\|^2 + \lambda\|\bw\|_1 + \frac{\kappa}{2}\|\bx - \bw\|^2
\end{eqnarray}
where $\bA\in\mathbb{R}^{m \times n}$ with $m \ge n$.
LASSO is often used for variable selection,
i.e. finding a sparse 
set of coefficients $\bx$ that correspond 
to variables (columns of $\bA$)
most useful for predicting the observation $\bb$.
We compare the quality and numerical efficiency of Eqs.~(\ref{eq:lasso})
and (\ref{eq:lasso_sr3}). 
The formulation in (\ref{eq:lasso_sr3}) is related to an earlier sequentially thresholded least square algorithm that was used for variable selection to identify nonlinear dynamical systems from data~\cite{Brunton2016pnas}.  

In all LASSO experiments, observations are generated by
$\bb = \bA \bx_t + \sigma \bm \epsilon$, where
$\bx_t$ is the true signal, and $\bm \epsilon$
is independent Gaussian noise. 

\begin{figure}[h]
\centering
\includegraphics[width=0.48\textwidth]{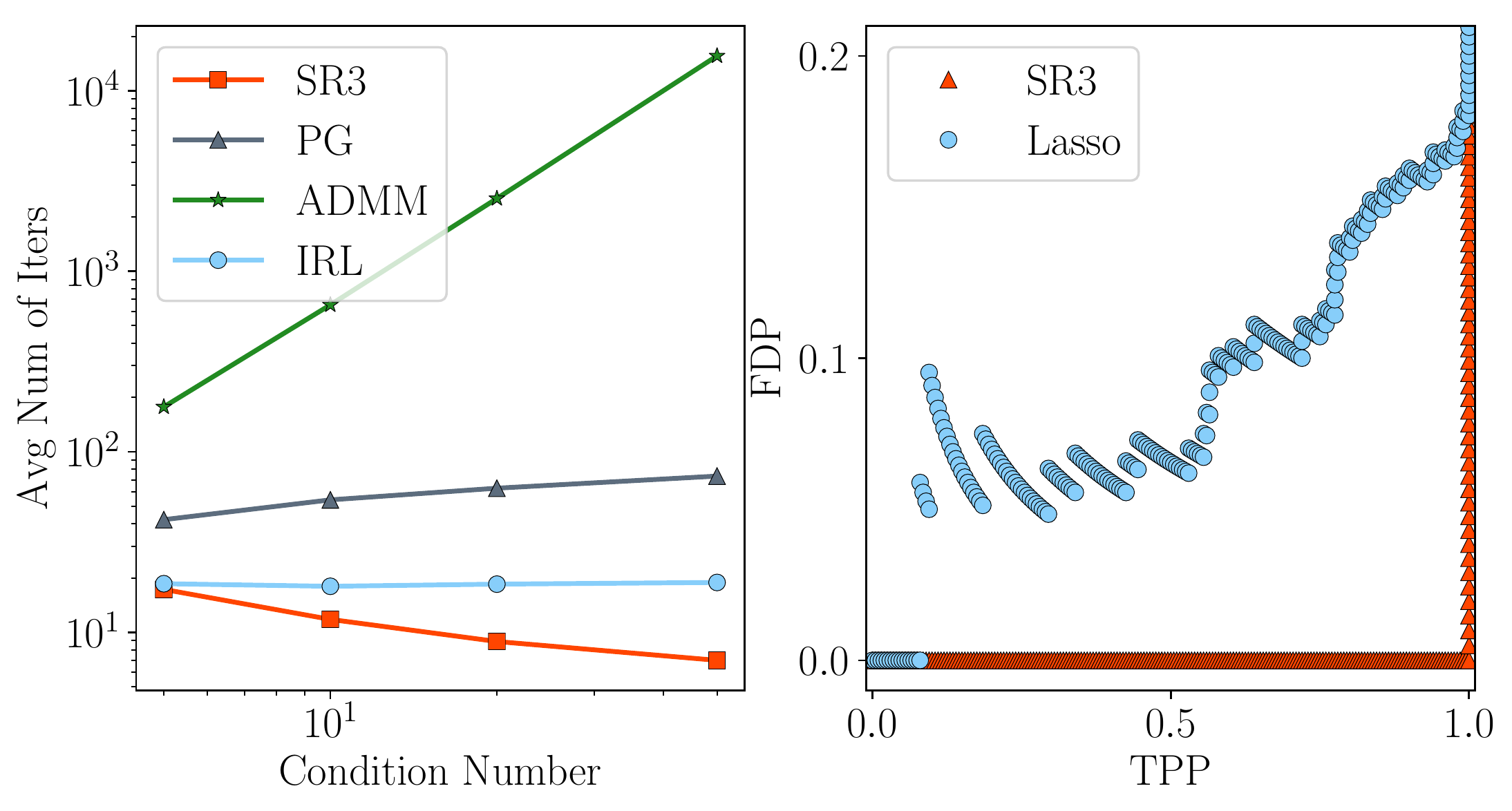}
\includegraphics[width=0.48\textwidth]{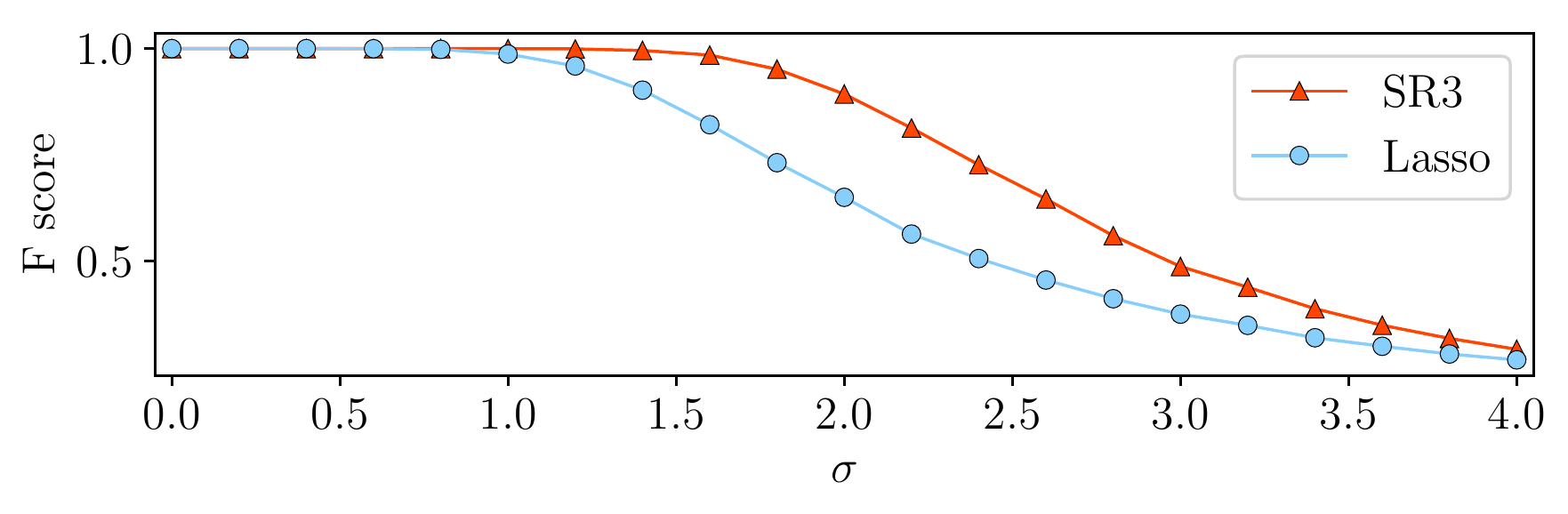}
\caption{{\bf Top Left:} \SR~approach (red) is orders of magnitude faster than ADMM (green) or other first-order methods such as prox-gradient (gray).
\blue{While IRL (blue) requires a comparable number of iterations, its cost per iteration is more expensive than \SR.}
{\bf Top Right:} True Positives vs. False Positives along the LASSO path (blue) and along the \SR~path (red). 
{\bf Bottom:} $F_1$ score of \SR~(red) and LASSO formulation (blue) with respect to different noise levels.}
\label{fig:variable_select}
\end{figure}

\subsubsection{LASSO Path.}
The LASSO path refers to the set of solutions obtained by sweeping over $\lambda$ in \eqref{eq:generic}
from a maximum $\lambda$, which gives $\bx = \bm 0$, down to $\lambda = 0$, which gives the least squares solution.
In \cite{su2017false}, it was shown that \eqref{eq:lasso} makes mistakes early along this path.

\noindent
\blue{{\bf Problem setup.}} As in \cite{su2017false}, the measurement matrix $\bA$ is $1010\times1000$, with entries drawn from $\cN(0,1)$.
The first 200 elements of the true solution $\bx_t$ are set to be 4 and the rest to be 0; $\sigma = 1$ is used to generate $\bb$.
Performing a $\lambda$ sweep, we track the fraction of incorrect nonzero elements in the last 800 entries vs. the fraction of nonzero elements in the first 200 entries 
of each solution, i.e. the false discovery proportion (FDP) and true positive proportion (TPP).

\noindent
\blue{{\bf Parameter selection.}
We fix $\kappa=100$ for \SR. Results are presented across a $\lambda$-sweep for both \SR~and LASSO. 
}

\noindent
\blue{{\bf Results.}}
The results are shown in the top-right panel of Fig.~\ref{fig:variable_select}.
LASSO makes mistakes early along the path~\cite{su2017false}.
In contrast, \SR~recovers the support without introducing
any false positives along the entire path until overfitting
sets in with the 201\textsuperscript{st} nonzero entry.
%

\subsubsection{Robustness to Noise.}
%
Observation noise makes signal recovery more difficult.
%
%
%
%
We conduct a series of experiments to compare the robustness with respect to noise of \SR~with LASSO.

\noindent
\blue{{\bf Problem setup.} We choose our sensing matrix with dimension $200$ by $500$ and elements drawn independently from a standard Gaussian distribution.
The true sparse signal has $20$ non-zero entries, 
and we consider a range of noise levels $\sigma\in\{0.2i:i = 0, 1, \ldots, 20 \}$.
For each $\sigma$, we solve \eqref{eq:lasso} and \eqref{eq:lasso_sr3} for 200 different random trials.
We record the $F_1$-score, $F_1=2(\mbox{precision} \cdot\mbox{recall})/(\mbox{precision} + \mbox{recall})$, to compare reconstruction quality.
In the experiments, any entry in $\bx$ which is greater than 0.01
is considered non-zero for the purpose of defining the recovered
support.}
%

\noindent
\blue{{\bf Parameter selection.} We FIX $\kappa = 100$ and perform a $\lambda$-sweep for both \eqref{eq:lasso} and \eqref{eq:lasso_sr3}
to record the best $F_1$-score achievable by each method.}

\noindent
\blue{{\bf Results.}
We plot the average normalized $F_1$-score for different noise levels in the bottom panel of Fig.~\ref{fig:variable_select}.
\SR~has a uniformly higher $F_1$-score across all noise levels.}

\subsubsection{Computational Efficiency.}
%
%
We compare the computational efficiency of the Alternating Directions Method of Multipliers (ADMM)
(see e.g.~\cite{boyd2011distributed,goldstein2009split}), proximal gradient algorithms (see e.g.~\cite{combettes2011proximal}) on \eqref{eq:lasso} with Algorithm~\ref{alg:pg_x_gen},
\blue{and a state-of-the-art Iteratively Reweighted Least-Squares (IRL) method, specifically IRucLq-v as in \cite{lai2013improved}.}

\noindent
\blue{{\bf Problem setup.} We generate the observations with $\sigma=0.1$.
%
%
The dimension of $\bA$ is $600\times 500$, and we vary the condition number of the matrix $\bA$ from 1 to 100.
For each condition number, we solve the problem 10 times and record the average number of iterations required to reach a specified tolerance.
We use the distance between the current and previous iteration to detect convergence for all algorithms.
When the measure is less than a tolerance of $10^{-5}$ we terminate the algorithms.}

%
\noindent
\blue{{\bf Parameter selection.}
  We choose $\kappa = 1$, $\lambda$ in \eqref{eq:lasso} to be $\|\bA^\top\bb\|_\infty/5$, and $\lambda$ in \eqref{eq:lasso_sr3}
  to be $\|\bF_\kappa^\top \bg_\kappa\|_\infty/5$.}

\begin{table}[h!]
\blue{
\caption{\label{table:comp}Complexity Comparison for $A \in \mathbb{R}^{m\times n}$, $m \geq n$.}
\begin{tabular}{c|c|c}
{\bf Method} & {\bf One-time Overhead} & {\bf Cost of generic iteration} \\\hline
PG &  --- & $O(mn)$ \\
ADMM & $O(mn^2 + n^3)$ & $O(n^2)$\\
IRucLq-v & --- & $O(mn^2 + n^3)$ \\ 
\SR &  $O(mn^2 + n^3)$ & $O(n^2)$ \\\hline
\end{tabular}
}
\end{table}

\noindent
\blue{{\bf Results.}
The results (by number of iterations) are shown in the top left panel of Fig.~\ref{fig:variable_select}.
The complexity of each iteration is given in Table~\ref{table:comp}. The generic iterations of PG, ADMM, and \SR~have 
nearly identical complexity, with ADMM and \SR~requiring a one-time formation and factorization of an $n\times n$ matrix. 
The IRucLq-v method requires the formation and inversion of such a matrix at each iteration. 
}
\blue{
From Fig.~\ref{fig:variable_select}, 
\SR~requires far fewer iterations than ADMM and the proximal gradient method, especially as $\cond(\bA)$ increases.
\SR And the IRucLq-v method require a comparable number of iterations.
%
A key difference is that ADMM requires dual variables, while \SR~is fundamentally a primal-only method. 
When $\cond(\bA) = 50$, ADMM needs almost $10^4$ iterations to solve~\eqref{eq:lasso};
proximal gradient descent requires $10^2$ iterations; and \SR~requires 10 to solve \eqref{eq:lasso_sr3}.
Overall, the \SR~method takes by far the least total compute time as the condition number increases. 
More detailed experiments, including for larger systems where iterative methods are needed, are left to future work. 
}

\subsubsection{\SR~for Compressed Sensing.}
\label{sec:cs}
%
When $m \ll n$, the variable selection problem
targeted by \eqref{eq:lasso} is often called {\it compressed sensing} (CS).
Sparsity is required to make the problem well-posed, as \eqref{eq:lasso} has infinitely many solutions with $\lambda=0$.
In CS, columns of $\bA$ are basis functions,
e.g. the Fourier modes $A_{ij} \!=\! \exp(\bm{i}\alpha_jt_i)$, and $\bb$ may be corrupted by noise~\cite{candes2006robust}.
In this case, 
compression occurs when
$m$ is smaller than the number of samples required
by the Shannon sampling theorem.

Finding the optimal sparse solution is inherently
combinatorial, and brute force solutions are only
feasible for small-scale problems.
In recent years, a series of powerful theoretical tools
have been developed in \cite{candes2005decoding,
  candes2006robust, candes2006stable,
  donoho2006compressed, donoho2009observed}
to analyze and understand the behavior of \eqref{eq:generic}
with $R(\cdot)=\|\cdot\|_1$ as a sparsity-promoting penalty. 
The main theme of these works is that if there is sufficient
incoherence between the measurements and the basis, then exact recovery
is possible.
One weakness of the approach is that the incoherence requirement --- for instance, having a small
restricted isometry constant (RIC) \cite{candes2005decoding} ---
may not be satisfied by the given samples, leading to sub-optimal
recovery.
%
%

%

%
\noindent
\blue{{\bf Problem setup.}
We consider two synthetic CS problems.
The sparse signal has dimension $d=500$ and
$k=20$ nonzero coefficients with uniformly distributed positions
and values randomly chosen as $-2$ or $2$.
In the first experiment, the entries of 
$\bA \in \mathbb{R}^{m \times n}$ are drawn
independently from a normal distribution, which
will generally have a small RIC \cite{candes2005decoding}
for sufficiently large $m$.
In the second experiment, entries of 
$\bA \in \mathbb{R}^{m\times n}$ are drawn from a
uniform distribution on the interval $[0,1]$, which are generally more coherent than using Gaussian
entries.}
\begin{figure*}
\centering
\begin{overpic}[width=0.49\textwidth]{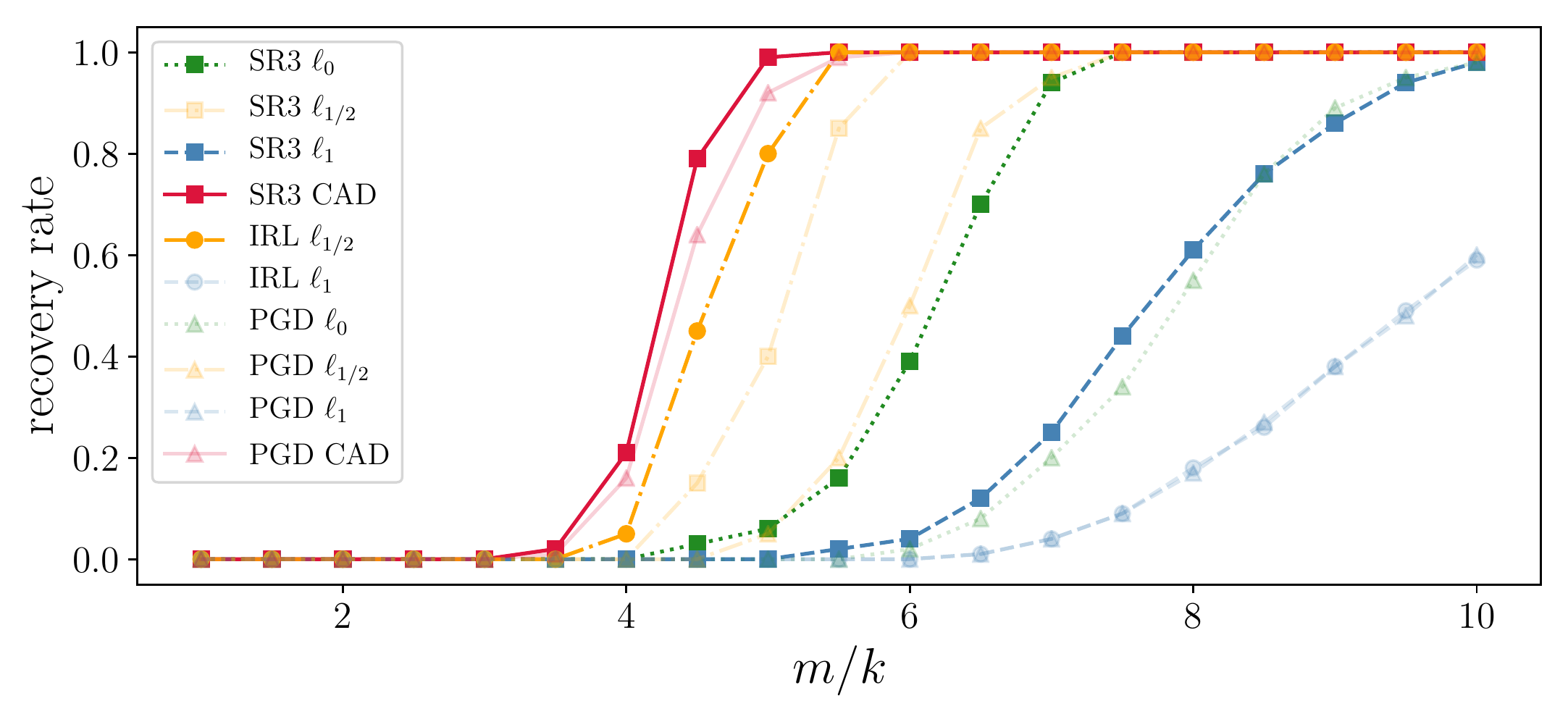} 
\put(33,45){Gaussian Sensing Matrix}
\end{overpic}
\begin{overpic}[width=0.49\textwidth]{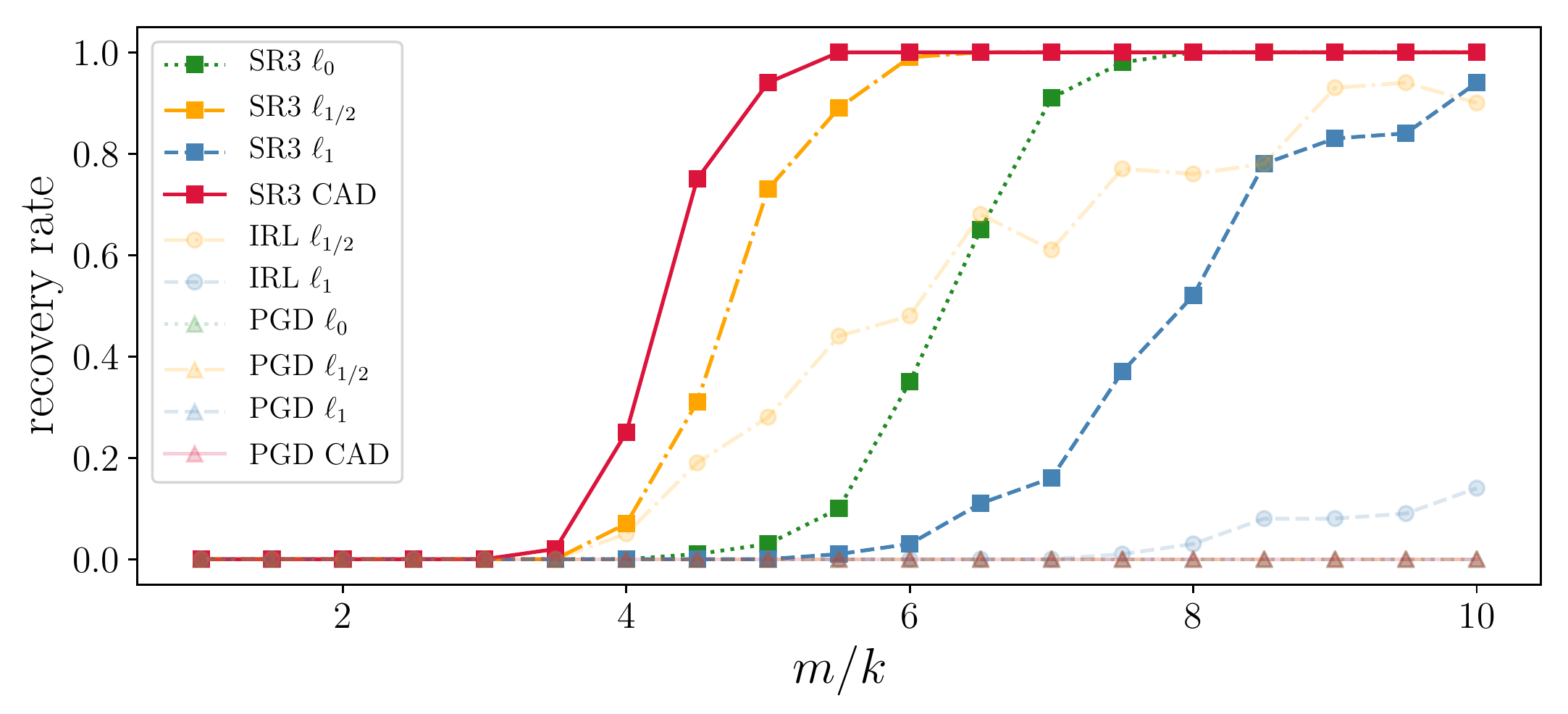} 
\put(33,45){Uniform Sensing Matrix}
\end{overpic}
\caption{\label{fig:cs}
\blue{Compressed sensing results: recovering a $20$-sparse signal
in $\mathbb{R}^{500}$ from a small number of measurements. We plot the
recovery rate as the number of measurements increases.
Line color and style are determined by the regularizer while
marker shapes are determined by the algorithm/formulation used.
For readability, only the best performing algorithm for each regularizer is
plotted in bold, with the rest opaque. {\bf Left panel:} the sensing matrix $\bA$
has {Gaussian} entries. 
Nonconvex regularizers are in general more effective than convex regularizers. 
\SR~is the most effective formulation for each
regularizer aside from $\ell_{1/2}$ for which the standard formulation with the IRucLq-v algorithm is best. 
\SR~CAD achieves a better final result compared to $\ell_{1/2}$ with IRLucLq-v. 
{\bf Right panel:}
the sensing matrix $\bA$ has {uniform} entries. The traditional convex approaches
fail dramatically as there is no longer a RIP-like condition. Even for
the nonconvex regularizers, IRucLq-v shows
significant performance degradation, while proximal gradient descent never succeeds.
However, \SR~approaches still succeed,
with only a minor efficiency gap (with respect to $m/k$) compared to the
easier conditions in the left panel. } }
\end{figure*}

\blue{
In the classic CS context, recovering the support of the signal
(indices of non-zero coefficients) 
is the main goal, as the optimal coefficients can be computed
in a post-processing step.
In the experiments, any entry in $\bx$ which is greater than 0.01
is considered non-zero for the purpose of defining the recovered
support.
To test the effect of the number of samples $m$ on recovery,
we take measurements  with additive Gaussian noise of the form $\cN(0,0.1)$,
and choose $m$ ranging from $k$ to $20k$.
For each choice of $m$ we solve \eqref{eq:generic} and
\eqref{eq:generalxw} 200 times.
We compare results from 10 different formulations and algorithms:
sparse regression with $\ell_0$, $\ell_{1/2}$, $\ell_1$ and CAD regularizers
using PG; 
\SR~reformulations of these four problems using Algorithm~\ref{alg:pg_x_gen},  
and sparse regression with $\ell_{1/2}$ and $\ell_1$ regularizers using IRucLq-v. 
}

\noindent
\blue{{\bf Parameter selection.}
For each instance, we perform a grid search on $\lambda$
to identify the correct non-zero support, if possible. 
The fraction of runs for which there is a $\lambda$ with successful
support recovery is recorded. 
For all experiments we fix $\kappa = 5$, and we set $\rho = 0.5$ for the CAD regularizer.} 

\noindent
\blue{{\bf Results.} As shown in Figure~\ref{fig:cs},
for relatively incoherent random Gaussian measurements,
both the standard formulation \eqref{eq:generic} and
\SR~succeed, particularly with the nonconvex regularizers.
$\mathrm{CAD}(\cdot,\rho)$, 
which incorporates some knowledge of the noise level in the parameter $\rho$, 
performs the best as a
regularizer, followed
by $\ell_{1/2}$, $\ell_0$, and $\ell_1$. The \SR~formulation
obtains a better recovery rate for each $m$ for most
regularizers, with the notable exception of $\ell_{1/2}$.
The IRucLq-v algorithm (which incorporates some knowledge
of the sparsity level as an internal parameter)
is the most effective method for
$\ell_{1/2}$ regularization for such matrices.}

\blue{For more coherent uniform measurements,
\SR~obtains a recovery rate which is only slightly degraded
from that of the Gaussian problem, while the results using
\eqref{eq:generic} degrade drastically.
In this case, \SR~is the most effective approach
for each regularizer and provides the only
methods which have perfect recovery at a sparsity
level of $m/k \leq 10$, namely \SR-CAD, \SR-$\ell_{1/2}$,
and \SR-$\ell_0$.
}

\noindent
\blue{
{\bf Remark:} Many algorithms focus on the noiseless setting in compressive sensing,
where the emphasis shifts to recovering signals that may have very small amplitudes~\cite{lai2013improved}. 
\SR~is not well suited to this setting, since the underlying assumption is that 
$\bw$ is near to $\bx$ in the least squares sense. 
}

\subsubsection{Analysis vs. Synthesis}

\blue{
Compressive sensing formulations fall into two broad categories, 
analysis \eqref{eq:analysis} and synthesis \eqref{eq:synthesis} (see \cite{chen2001atomic,elad2007analysis}):
\begin{align}
\min_{\bm x}~&\frac{1}{2}\|\bA \bx - \bb\|^2 + R(\bC \bx), \label{eq:analysis}\\
\min_{\bm \xi}~&\frac{1}{2}\|\bA\bC^\top \bm \xi - \bb\|^2 + R(\bm \xi), \label{eq:synthesis}
\end{align}
where $\bC$ is the {\it analyzing operator}, $\bx \in \mathbb{R}^d$ and $\bm \xi \in \mathbb{R}^n$,
and we assume $n \gg d$. In this section, we consider $\bC^\top \bC = \mbf I$, i.e. $\bC^\top$ is a tight frame. 
Synthesis represents $\bx$ using the over-determined system $\bC^\top$, and recovers the
coefficients $\bm \xi$ using sparse regression.
Analysis directly works over the domain of the underlying signal $\bx$ with the prior that $\bC \bx$
is sparse.
The two methods are equivalent when $n \le d$, and very different when $n > d$ \cite{chen2001atomic}.
Both forms appear in a variety of inverse problems including denoising, interpolation and super-resolution.
The work of \cite{elad2007analysis} presents a thorough comparison of \eqref{eq:analysis} and \eqref{eq:synthesis} 
across a range of signals, and finds that the effectiveness of each depends on problem type.
}

\blue{
The \SR~formulation can easily solve both analysis and synthesis formulations. 
We have focused on synthesis thus far, so in this section we briefly consider 
analysis \eqref{eq:analysis}, under the assumption that $\bC \bx$ is almost sparse.
When $l \gg d$, the analysis problem is formulated over a lower dimensional space.
However, since $\bC \bx$ is always in the range of $\bC$, it can never be truly sparse. 
If a sparse set of coefficients is needed, analysis formulations use post-processing steps such as thresholding. 
\SR, in contrast, can extract the sparse transform coefficients directly from the $w$ variable.
We compare \SR~with the Iteratively Reweighted Least-Squares-type algorithm
IRL-D proposed by \cite{huang2018new} for solving \eqref{eq:analysis}.
}

\noindent
\blue{{\bf Problem setup.}
We choose our dimensions to be $n=1024$, $d=512$ and $m=128$.
We generate the sensing matrix $\bA$ with independent Gaussian entries
and the true sparse coefficient $\bm \xi_t$ with 15 non-zero elements randomly selected
from the set $\{-1, 1\}$. The true underlying signal is $\bx_t = \bC^\top \bm \xi$ and the measurements are generated
by $\bb = \bA \bx_t + \sigma \bm\epsilon$, where $\sigma = 0.1$ and $\bm \epsilon$ has independent Gaussian entries.
We use $\ell_1$ as the regularizer, $R(\cdot) = \lambda\|\cdot\|_1$.
}

\noindent
\blue{{\bf Parameter selection.}
In this experiment, we set $\kappa$ for \SR~to be $5$, $\lambda$ for \SR~to be $\|\bF_\kappa^\top \bg_\kappa\|_\infty/2$,
and $\|\bA^\top\bb\|_\infty/10$ for IRL-D. The $\lambda$s are chosen to achieve the clearest separation between active 
and inactive signal coefficients for each method.
} 
\begin{figure}[h]
\centering
\includegraphics[width=0.49\textwidth]{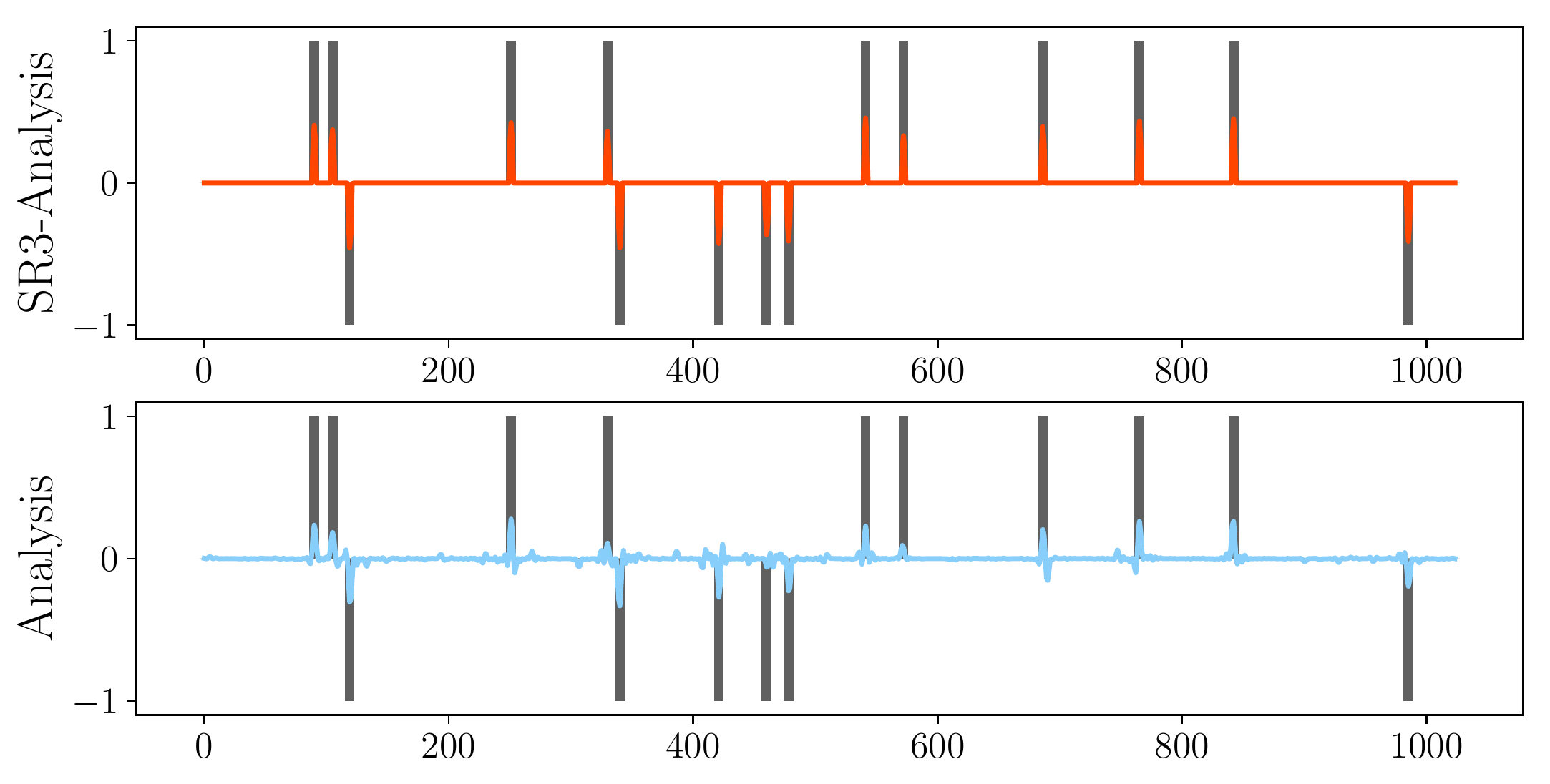}
\caption{\blue{Comparison of standard analysis with \SR-analysis.
{\bf Top panel:} result using \SR-analysis, plotting the final $\bw$ (red) against the true signal (dark grey).
{\bf Bottom panel:} result using standard analysis and the IRL-D algorithm, plotting final $\bC \bx$ (blue) against the true signal (dark grey).}}
\label{fig:tightframe}
\end{figure}

\noindent
\blue{{\bf Results.}
The results are shown in Figure~\ref{fig:tightframe}. The $\bw$ in the \SR~analysis formulation
is able to capture the support of the true signal cleanly, while $\bC \bx$ from the \eqref{eq:analysis} identifies the support 
but is not completely sparse, requiring post-processing steps such as thresholding to get a support estimate.
}

\subsection{\SR~for Total Variation Regularization}
\label{sec:TV}
Natural images are effectively modeled as 
large, smooth features separated by a few sparse edges.
It is common to regularize 
ill-posed inverse problems in imaging by adding the
so-called total variation (TV) regularization
\cite{rudin1992nonlinear,chan1998total,
  strong2003edge,osher2005iterative,wang2008new,
  beck2009fast,chan2011augmented}. 
Let $X_{ij}$ denote the $i,j$ pixel of an
$m\times n$ image. For convenience,
we treat the indices as 
doubly periodic, i.e. $X_{i+pm,j+qn} = X_{i,j}$ for $p,q\in \Z$.
Discrete $x$ and $y$ derivatives are defined by
$[\bD_x \bX]_{ij} = X_{i+1,j}-X_{ij}$ and
$[\bD_y \bX]_{ij} = X_{i,j+1}-X_{ij}$, respectively. The (isotropic)
total variation of the image is then given by the sum of
the length of the discrete gradient at each pixel, i.e.
\begin{equation}
  R_\TV\left(\begin{matrix} \bD_x\bX\\ \bD_y\bX \end{matrix}\right) :=
  \sum_{i=1}^m \sum_{j=1}^n \sqrt{ [ \bD_x \bX ]_{ij}^2 +
    [\bD_y\bX]_{ij}^2} \; . \label{eq:tvnorm}
\end{equation}
Adding the TV regularizer (\ref{eq:tvnorm}) to a
regression problem corresponds to imposing a sparsity
prior on the discrete gradient.

Consider image deblurring
(Fig.~\ref{fig:tv_cameraman}). 
The two-dimensional
convolution $\bY = \bA * \bX$ is given by the sum
\(
  Y_{ij}\! =\! \sum_{p=1}^m \!\sum_{q=1}^n \!A_{pq} X_{i-p,j-q} \; .
\)
\begin{figure}[t]
  \centering
\begin{overpic}[width=8cm]{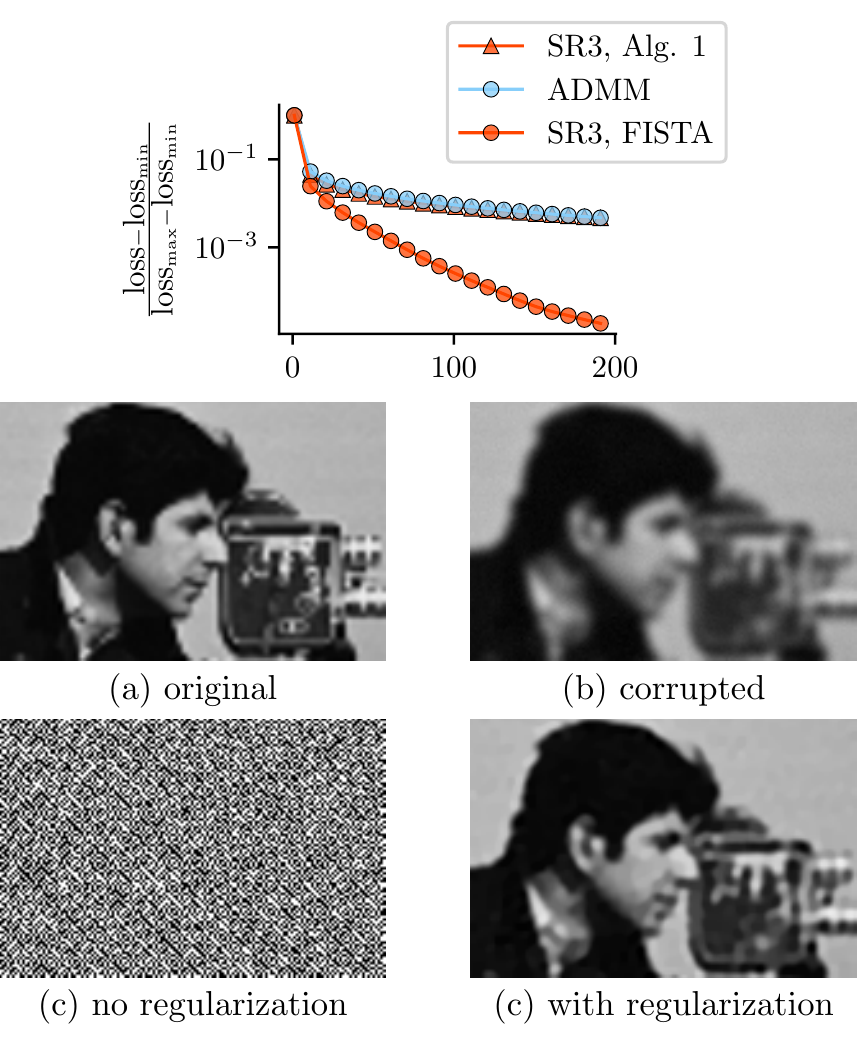}
\put(18,63.77){iters:}
    \end{overpic}
  \caption{\label{fig:tv_cameraman} The top plot compares
    the progress of the \SR~and ADMM-type algorithms in
    reducing their losses, showing similar rates of convergence.
    Panels (a) and (b) show a detail
    of the original cameraman image and the image corrupted as
    described in the text, respectively. The incredibly noisy
    image resulting from inverting the blur without regularization
    ($\lambda = 0$) is shown in panel (c) and the crisper
    image resulting from the
    regularized \SR~problem (with $\lambda = .075$) is shown in
    panel (d) (the image resulting from the ADMM type algorithm
    of \cite{chan2011augmented} is visually similar, with a
    similar SNR)}
\end{figure}
Such convolutions are often used to model photographic
effects, like distortion or motion blur.
Even when the kernel $\bA$ is known, the problem of recovering
$\bX$ given the blurred measurement
is unstable because measurement noise 
 is sharpened by `inverting' the blur. Suppose that $\bB = \bA * \bX + \nu \bG$,
where $\bG$ is a matrix with entries given by independent
entries from a standard normal distribution and $\nu$
is the noise level. To regularize the problem of recovering
$\bX$ from the corrupted signal $\bB$,
we add the TV regularization:
\begin{equation}
  \hat{\bX} = \argmin_\bX \frac{1}{2} \| \bA * \bX - \bB \|_F^2
  + \lambda R_\TV\left(\begin{matrix} \bD_x \bX\\ \bD_y \bX\end{matrix}\right) \; . \label{eq:imtv}
\end{equation}
The natural \SR~reformulation is given by 
\begin{align}
  \min_{\bX,\bw_x,\bw_y} &\frac{1}{2} \| \bA \!*\! \bX \!-\! \bB \|_F^2 \nonumber\\
 & \!+\! \lambda R_\TV\left(\begin{matrix}\bw_x \\\bw_y\end{matrix}\right) \!+\! \frac{\kappa}{2} \left\|\begin{matrix} \bw_x\!-\!\bD_x\bX\\  \bw_y\!-\!\bD_y\bX \end{matrix}\right\|_F^2 \label{eq:imtvsr3} .
\end{align}

\noindent
\blue{{\bf Problem setup.} In this experiment, we use} the standard Gaussian
blur kernel of size $k$ and standard deviation
$\sigma$, given by
\(
  A_{ij} = \exp \left ( -(i^2+j^2)/(2\sigma^2)
  \right ),
\)
when $|i| < k$ and $|j|<k$, with the rest of the entries of $\bA$
determined by periodicity or equal to zero. \blue{The signal
$\bX$ is the classic ``cameraman'' image of size $512\times 512$.}
\blue{As a measure of the progress of a given method
  toward the solution, we evaluate the current loss at each
  iteration (the value of either the right hand side of
  \eqref{eq:imtv} or \eqref{eq:imtvsr3}).}

\noindent
\blue{{\bf Parameter Selection.} We set $\sigma = 2$, $k=4$,
  $\nu = 2$, and $\lambda = 0.075$. The value of $\lambda$
  was chosen by hand to achieve reasonable image recovery.
  For \SR, we set $\kappa = 0.25$.}

\noindent
\blue{{\bf Results.}} Figure~\ref{fig:tv_cameraman} demonstrates the stabilizing
effect of TV regularization.
Panels (a) and (b) show a detail of the
image, i.e. $\bX$, and the corrupted image, i.e. $\bB$,
respectively.
In panel (c), we see that simply inverting the effect of the
blur results in a meaningless image. Adding TV regularization
 gives a more reasonable result in
panel (d).

\begin{algorithm}[h!]
\caption{FISTA for \SR~TV}
\label{alg:accl_pg}
\begin{algorithmic}[1]
\State {\bfseries Input:} $\bw^0$
\State {\bfseries Initialize:} $k=0$, $a_0 = 1$, $\bv_0 = \bw^0$, $\eta \leq \frac{1}{\kappa}$
\While{not converged}
\Let{$k$}{$k+1$}
\Let{$\bv_k$}{$\prox_{\eta R}(\bw^{k-1} - \eta (\bF_\kappa^\top (\bF_\kappa \bw^{k-1} - \bg_\kappa)))$}
\Let{$a_k$}{$(1+\sqrt{1+4a_{k-1}^2})/2$}
\Let{$\bw^k$}{$\bv_k + (a_{k-1}-1)/a_k(\bv_{k} - \bv_{k-1})$}
\EndWhile
\State {\bfseries Output:} $\bw^k$
\end{algorithmic}
\end{algorithm}
In the top plot of Fig.~\ref{fig:tv_cameraman}, we compare \SR~and
a primal-dual algorithm \cite{chan2011augmented} on the objectives
\eqref{eq:imtvsr3} and \eqref{eq:imtv}, respectively.
Algorithm~\ref{alg:pg_x_gen} converges
as fast as the state-of-the-art method of~\cite{chan2011augmented}; 
it is not significantly faster because for TV regularization, the equivalent of the map $\bC$
\blue{does not have orthogonal columns (so that the stronger
  guarantees of Section~\ref{sec:method} do not apply)} and
the equivalent of $\bF_\kappa$, see \eqref{eq:Fdef}, is still ill-conditioned.
Nonetheless, since \SR~gives a primal-only method, it is straightforward
to accelerate using FISTA~\cite{beck2009fastlin}. In Fig.~\ref{fig:tv_cameraman},
we see that this accelerated method converges much more rapidly to
the minimum loss, giving a significantly better algorithm for TV deblurring.
The FISTA algorithm for \SR~TV is detailed in Algorithm~\ref{alg:accl_pg}.

\blue{We do not compare the support recovery of the two
  formulations, \eqref{eq:imtv} and \eqref{eq:imtvsr3},
  because the original signal does not have
  a truly sparse discrete gradient. The recovered signals
  for either formulation have comparable signal-to-noise
  ratios (SNR), approximately 26.10 for \SR~and 26.03
  for standard TV (these numbers vary quite a bit based
  on parameter choice and maximum number of iterations).}

\noindent
\blue{{\bf Analysis.}} We can further analyze \SR~for the \blue{specific} $\bC$
used in the TV denoising problem
in order to understand the mediocre performance of
unaccelerated \SR.
Setting $\bx = \vectorize(\bX)$, we have
\[
\begin{aligned}
\bA * \bX &= \cF^{-1} \Diag(\hat\bc) \cF \bx, \quad \bD_x \bX = \cF^{-1}\Diag(\hat\bd_x)\cF \bx, \\
\bD_y \bX &= \cF^{-1}\Diag(\hat\bd_y)\cF \bx
\end{aligned}
\]
where $\cF\bx$ corresponds to taking a 2D Fourier transform, i.e. of $\cF\bx = \vectorize(\cF^{(2\mathrm{d})}\bX)$.
Then, $\bF_\kappa$ can be written as
\[
\small
\begin{bmatrix}
\kappa \cF^{-1} \Diag(\hat\bc) \bH_\kappa^{-1}\begin{bmatrix} \Diag(\hat\bd_x) & \Diag(\hat \bd_y)\end{bmatrix} \cF\\ \\
\sqrt{\kappa}\cF^{-1}\left(\mathbf I - \kappa \begin{bmatrix} \Diag(\hat\bd_x) \\ \Diag(\hat\bd_y) \end{bmatrix} \bH_\kappa^{-1} \begin{bmatrix} \Diag(\hat\bd_x) & \Diag(\hat \bd_y)\end{bmatrix}\right )\cF
\end{bmatrix} \; ,
\]
where
\[
\bH_\kappa = \cF^{-1}\Diag(\hat \bc \odot \hat \bc + \kappa\hat \bd_x \odot \hat \bd_x + \kappa\hat \bd_y \odot \hat \bd_x) \cF,
\]
and $\odot$ is element-wise multiplication.
The \SR~formulation \eqref{eq:imtvsr3}  reduces to 
\[
\min_{\bw} \frac{1}{2}\|\bF_\kappa \bw - \bg_\kappa\|^2 + \lambda\|\bw\|_1,
\]
with $\bF_\kappa$ and $\bg_\kappa$ as above, and  
\(
\bw = \vectorize \left (\circ \sqrt{\bW_x^\elsq + \bW_y^\elsq } \right ) \; ,
\)
where $\circ\sqrt{A}$ and $A^\elsq$ denote element-wise
square root and squaring operations, respectively.

Setting $\hat\bh = \hat \bc \odot \hat \bc + \kappa\hat \bd_x \odot \hat \bd_x + \kappa\hat \bd_y \odot \hat \bd_x$, we have
\[
\bF_\kappa^\top \bF_\kappa = \cF^{-1}\cA_\kappa\cF,
\]
with $\cA_\kappa$ given by 
\[
\small
\begin{bmatrix}
\kappa \mathbf I - \kappa^2 \Diag(\hat\bd_x \odot \hat \bh^{-1} \odot \hat \bd_x) & - \kappa^2 \Diag(\hat\bd_x \odot \hat \bh^{-1} \odot \hat \bd_y)\\
- \kappa^2 \Diag(\hat\bd_y \odot \hat \bh^{-1} \odot \hat \bd_x) & \kappa \mathbf I - \kappa^2 \Diag(\hat\bd_y \odot \hat \bh^{-1} \odot \hat \bd_y)
\end{bmatrix}.
\] 
$\bF_\kappa^\top \bF_\kappa$ is a $2\times 2$ block system of diagonal matrices,
so we can efficiently compute its eigenvalues, thereby obtaining the singular values of $\bF_\kappa$.
In Figure~\ref{fig:tv_spectrum}, we plot the spectrum of
$\bF_\kappa$.
\begin{figure}[t]
\centering
\includegraphics[width=0.5\textwidth]{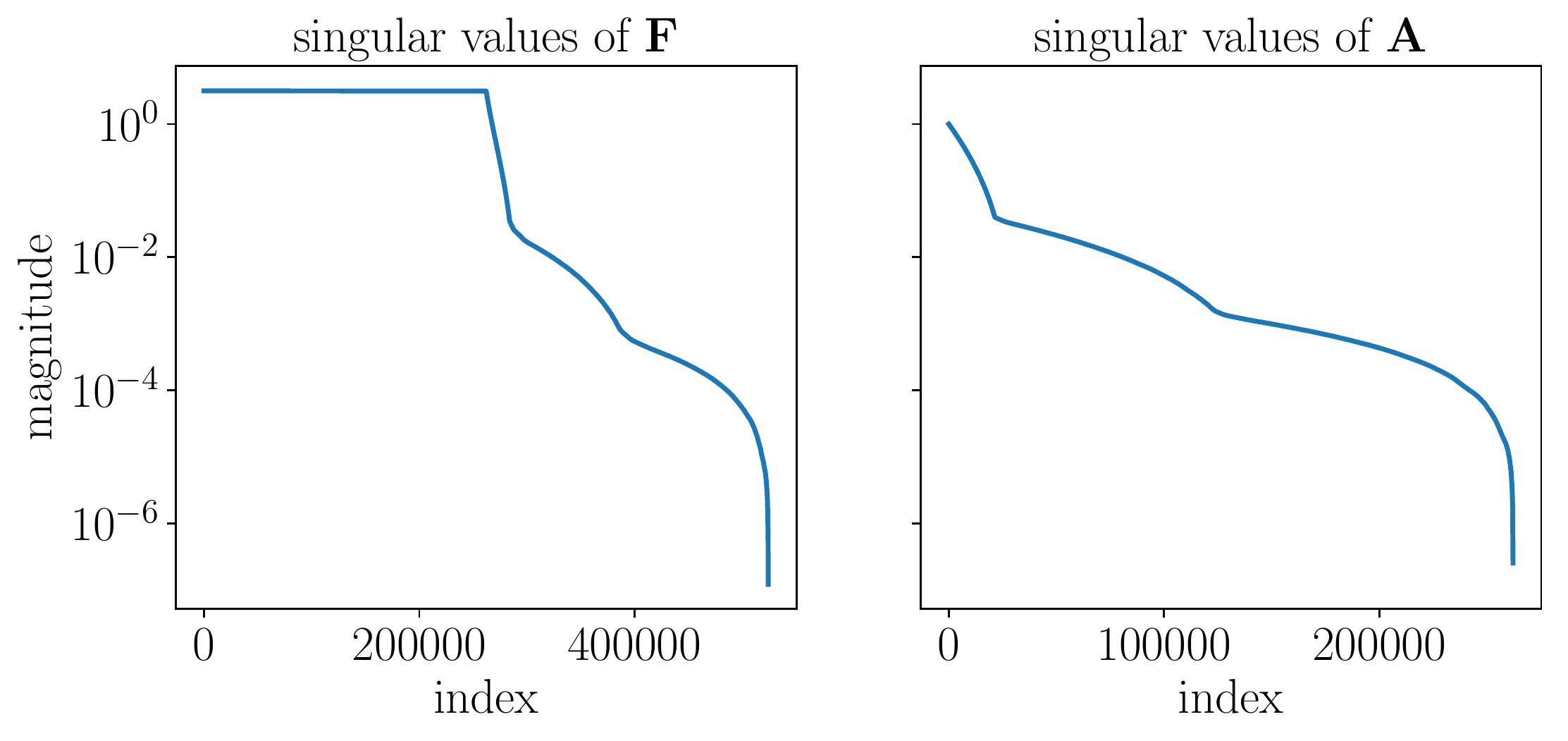}
\caption{Singular values (ordered by magnitude) of $\bF_\kappa$ (left panel) and $\bA$ (right panel) in the TV example.}
\label{fig:tv_spectrum}
\end{figure}
Half of the singular values are exactly $\sqrt{\kappa}$, and the other half drop rapidly to 0.
This spectral property is responsible for the slow sublinear convergence rate
of SR3. Because of the special structure of the $\bC$ matrix, $\bF_\kappa$ does not improve 
conditioning as in the LASSO example, where $\bC= \bI$. The \SR~formulation still makes it simple to apply the FISTA algorithm to 
the reduced problem~\eqref{eq:valueExp}, improving the convergence rates.

\subsection{\SR~for Exact Derivatives}
TV regularizers are often used in physical settings,
where the position and the magnitude of the
non-zero values for the derivative matters.
In this numerical example, we use synthetic data to
illustrate the efficacy of \SR~for such problems.
In particular, we demonstrate that the use of nonconvex
regularizers can improve performance.
\noindent
\blue{{\bf Problem setup.}
Consider a piecewise constant step function with dimension $\bx_t\in\mathbb{R}^{500}$ and values from $-2$ to $2$, see the first row of
Figure~\ref{fig:syn_tv} for a sample plot.
We take $100$ random measurements $\bb = \bA \bx_t + \sigma\bm \epsilon$ of the signal, where the elements of
$\bA$ and $\bm \epsilon$ are i.i.d. standard Gaussian, and we choose a noise level of $\sigma=1$.}

\begin{figure}[h]
\centering
\includegraphics[width=0.49\textwidth]{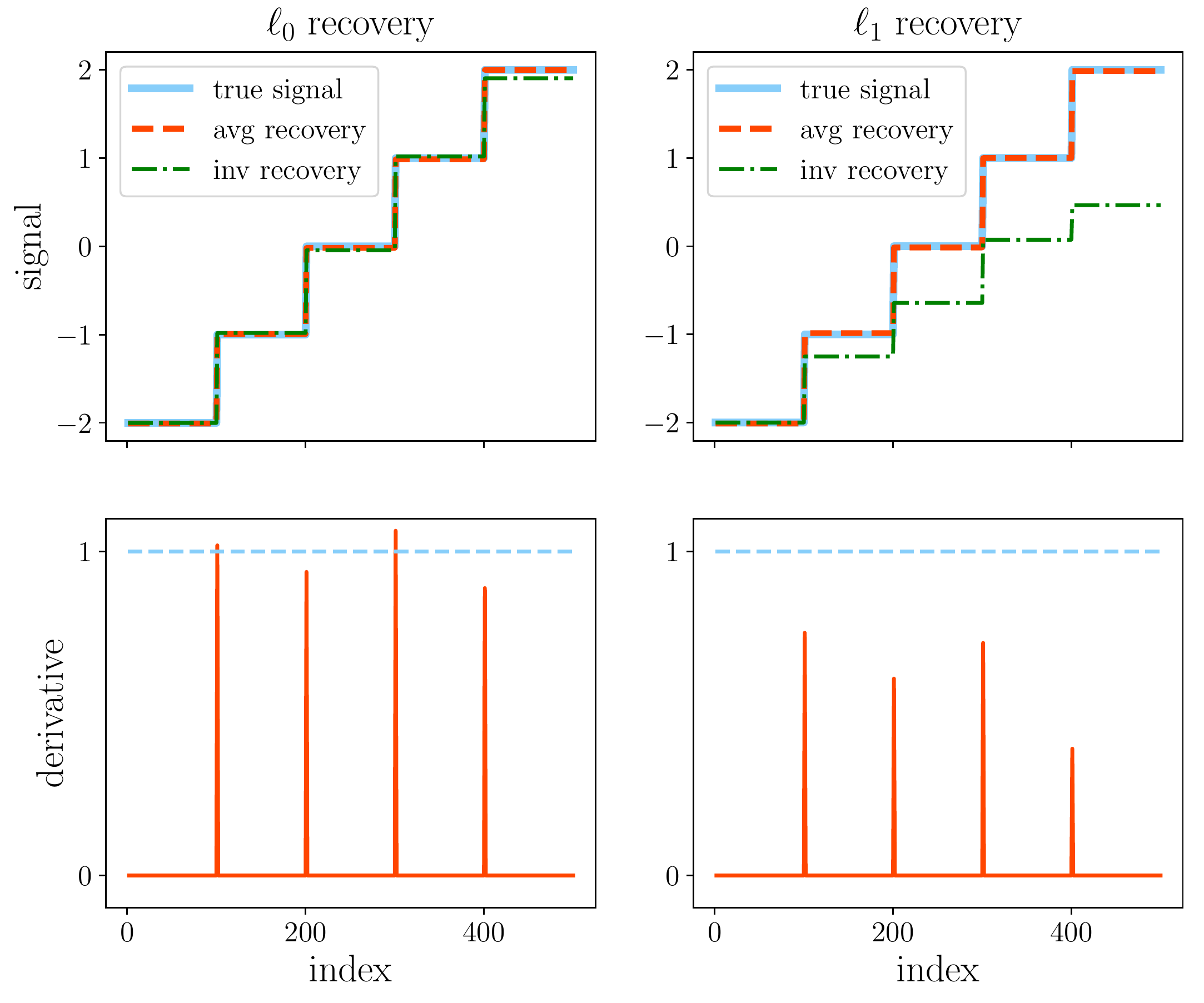}
\caption{SR3 TV regularization result on synthetic data. The first row plots the averaging recovery signal (dashed red line), integrating recovery signal (dot dashed green line) and
the true signal (solid blue line). Second row plots the discretized derivative (solid red line) and true magnitude (dashed blue line).
First column contain the results come from $\ell_0$ regularization, second column is from $\ell_1$.}
\label{fig:syn_tv}
\end{figure}

To recover the signal, we solve the \SR~formulation
\[
\min_{\bx, \bw} \frac{1}{2}\|\bA \bx - \bb\|^2 + \lambda R(\bw) + \frac{1}{2}\|\bw - \bC \bx\|^2,
\]
where $R$ is chosen to be $\|\cdot\|_0$ or $\|\cdot\|_1$, and $\bC$ is the appropriate forward difference matrix.
We want to both recover the signal $\bx_t$ and obtain
an estimate of the discrete derivative using $\bw$.

\noindent
\blue{{\bf Parameter selection.}
We set $\kappa=1$ and choose $\lambda$ by cross-validation. 
We set $\lambda=0.07$ when $R = \ell_1$ and $\lambda=0.007$
when $R = \ell_0$.
}

\noindent
\blue{{\bf Results.}
Results are shown in Figure~\ref{fig:syn_tv}, with the first row showing the recovered signals (red dashed line and green dot-dashed line) vs. true signal (blue solid line) 
and the second row showing the estimated signal derivative $\bw$.}
If we explicitly use the fact that our signal is a step function,
it is easy to recover an accurate approximation of the
signal using both $\bx$ and $\bw$.
We define groups of indices corresponding to contiguous
sequences for which $w_i = 0$. For such contiguous groups,
we set the value of the recovered signal to be the mean
of the $x_i$ values.
Ideally, there should be five such groups.
In order to recover the signal, we need good group identification (positions of nonzeros in $\bw$) and an unbiased estimation for signal $\bx$.
From the red dash line in the first row of Figure~\ref{fig:syn_tv}, we can see that both $\ell_0$ and $\ell_1$ reasonably achieve this goal using the grouping procedure.
However, such an explicit assumption on the structure of
the signal may not be appropriate in more complicated
applications. A more generic approach would ``invert'' $\bC$
(discrete integration in this example) to reconstruct the
signal given $\bw$.
From the second row of Figure~\ref{fig:syn_tv} we see that $\ell_0$-TV obtains a better unbiased estimation of the magnitude of the derivative compared to $\ell_1$-TV; accordingly, the signal
reconstructed by integration is more faithful using the
$\ell_0$-style regularizatoin.

\begin{figure}[t]
\centering
\includegraphics[width=0.4\textwidth]{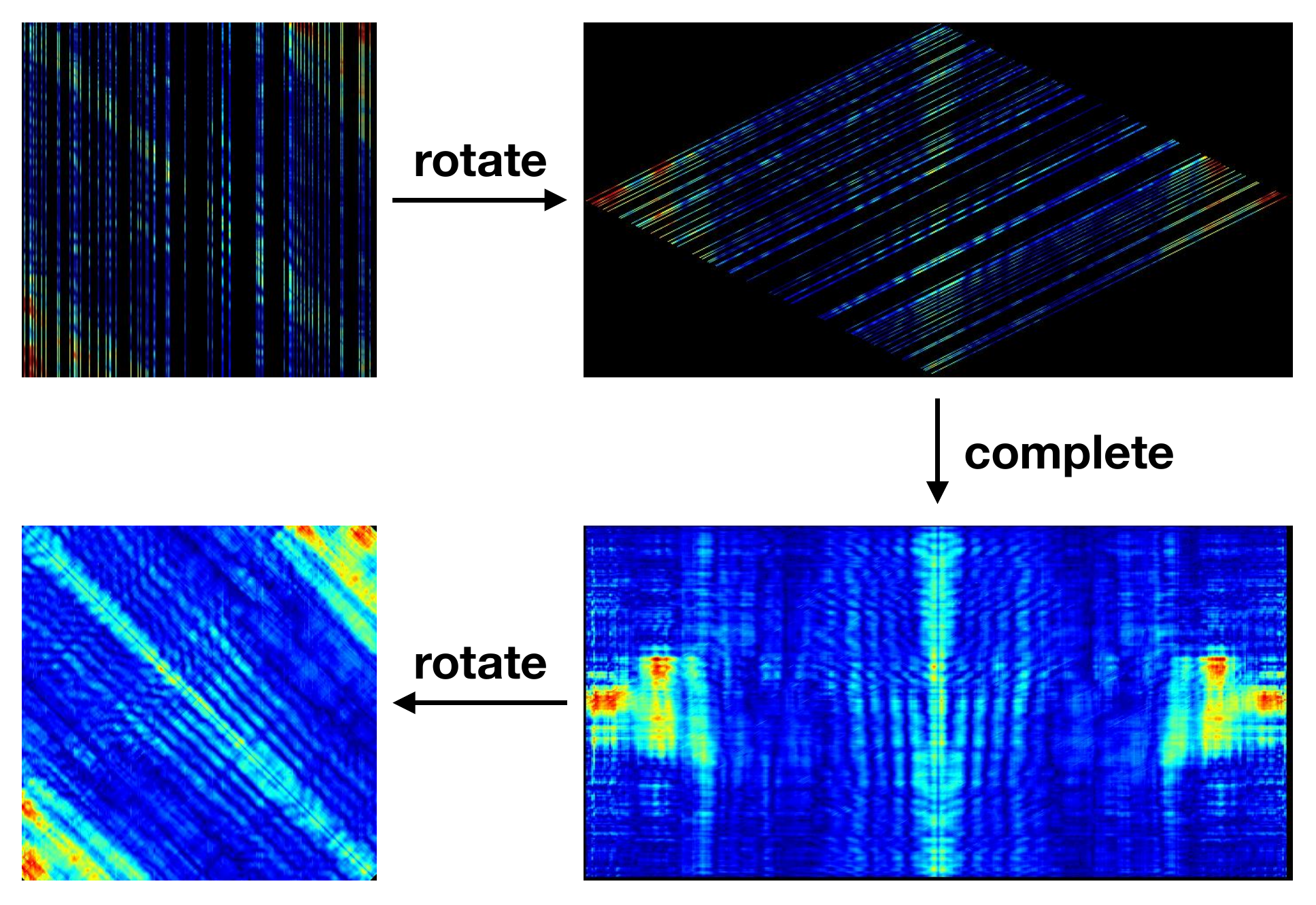}
\caption{\label{fig:mat_data} Interpolating a frequency slice from the Gulf of Suez dataset. Clockwise 
we see subsampled data in the source-receiver domain; transformation of the data to the midpont-offset domain, 
interpolation, and inverse transform back to the source/receiver domain.}
\end{figure}

\subsection{\SR~for Matrix Completion}
Analogous to sparsity in compressed sensing, low-rank structure has been used to solve a variety of matrix completion problems, including the famous Netflix Prize problem, as well as in control, 
system identification, {signal processing~\cite{Yang2017}}, combinatorial
optimization~\cite{RechtFazelParrilo2010,Candes2011-JACM}, 
and seismic data interpolation/denoising~\cite{oropeza:V25, aravkin2014fast}. 

We compare classic rank penalty approaches using the 
nuclear norm (see e.g.~\cite{RechtFazelParrilo2010}) to the \SR~approach 
on a seismic interpolation example. 
Seismic data interpolation is crucial for accurate inversion and imaging procedures such as full-waveform inversion \cite{virieux2009overview}, reverse-time migration \cite{baysal1983reverse} and multiple removal methods \cite{verschuur1992adaptive}.
Dense acquisition is prohibitively expensive in these applications, motivating reduction in seismic measurements. On the other hand, using subsampled sources and receivers without interpolation gives unwanted imaging artifacts.
The main goal is to simultaneously sample and compress a signal using optimization to replace dense acquisition, thus enabling a range of applications in seismic data processing at a fraction of the cost.
\begin{figure*}[h!]
\centering
\begin{tabular}{c  c}
\includegraphics[width=0.2\textwidth]{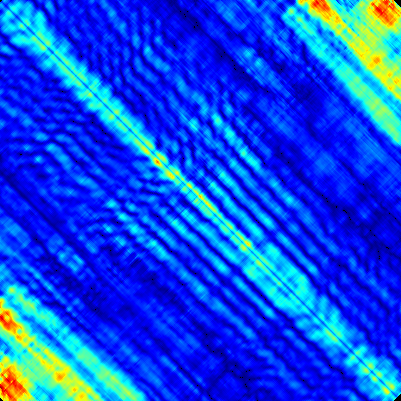} \includegraphics[width=0.2\textwidth]{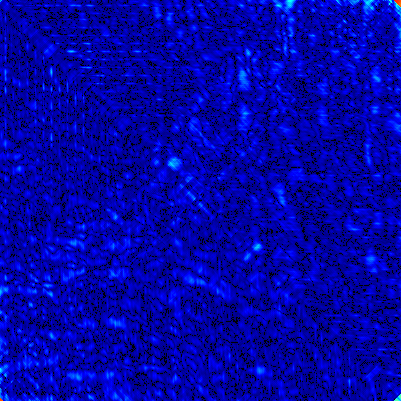} &
\includegraphics[width=0.2\textwidth]{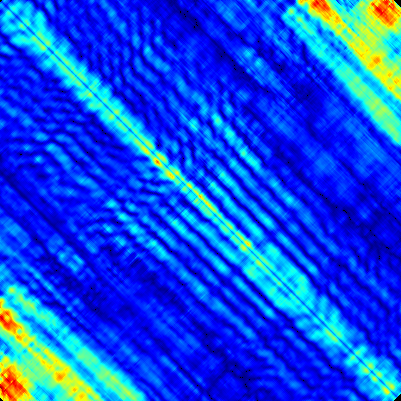} \includegraphics[width=0.2\textwidth]{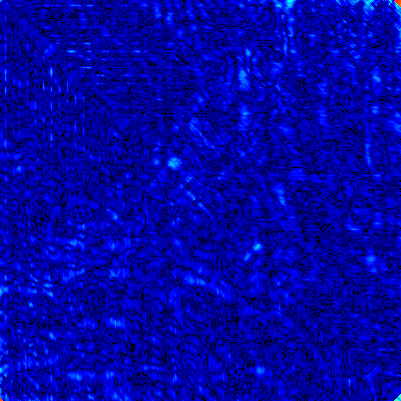} \\
(a) \SR~\eqref{eq:mat_sr3}, $R = \|\cdot\|_0$, SNR: $12.6489$ & (b) \SR~\eqref{eq:mat_sr3}, $R = \|\cdot\|_1$, SNR: $12.3508$\\
\includegraphics[width=0.2\textwidth]{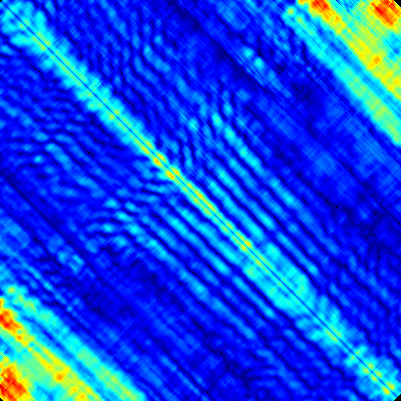} \includegraphics[width=0.2\textwidth]{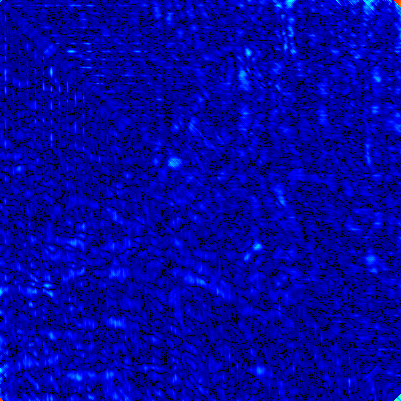} &
\includegraphics[width=0.2\textwidth]{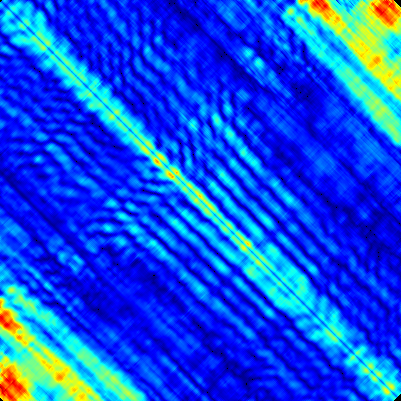} \includegraphics[width=0.2\textwidth]{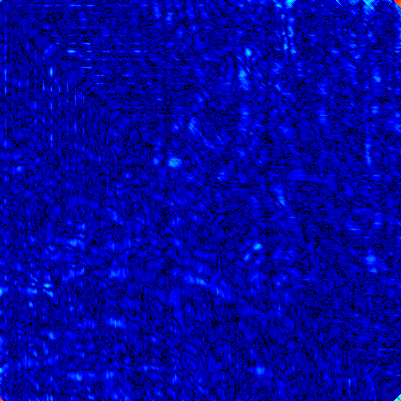} \\
(c) \eqref{eq:mat_pg}, $R = \|\cdot\|_0$, SNR: $12.1929$ & (d) \eqref{eq:mat_pg}, $R = \|\cdot\|_1$, $12.0572$
\end{tabular}
\includegraphics[width=0.5\textwidth]{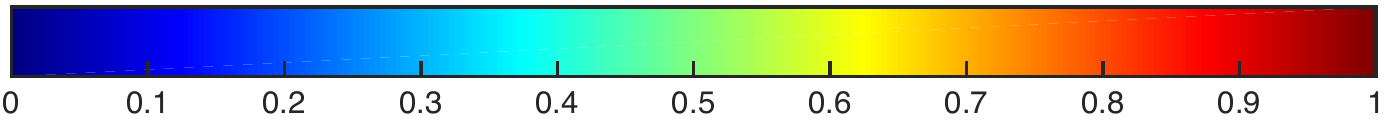}
\caption{Result comparison SR3 vs. classic low rank regression. In each subplot, we show the recovered signal matrix (left) and the difference between recovered the true signal (right).
The corresponding SNR is provided.
(a), (b) plot the the results of SR3 with $\ell_0$ and $\ell_1$ regularizers. (c), (d) plot the results of classic formulation with $\ell_0$ and $\ell_1$ regularizers.}
\label{fig:mat_result}
\end{figure*}

\noindent
\blue{{\bf Problem setup.}
We use a real seismic line from the Gulf of Suez. The signal is stored in a $401\times 401$ complex matrix, 
arranged as a matrix by source/receiver, see the left plot of Fig.~\ref{fig:mat_data}.}
Fully sampled seismic data has a fast decay of singular values, while sub-sampling breaks this decay~\cite{aravkin2014fast}.
A convex formulation for matrix completion with nuclear norm is given by~\cite{RechtFazelParrilo2010}
\begin{equation}
\label{eq:mat_pg}
\min_{\bX} \frac{1}{2}\|\mathcal{A}(\bX) - \bD\|_F^2 + \lambda R(\sigma(\bX))
\end{equation}
where $\mathcal{A}$ maps $\bX$ to data $\bD$, and  
$R(\cdot) = \|\cdot\|_1$ penalizes rank. 

The \SR~model relaxes~\eqref{eq:mat_pg} to obtain the formulation 
\begin{equation}
\label{eq:mat_sr3}
\min_{\bX, \bW} \frac{1}{2}\|\mathcal{A}(\bX)  - \bD\|_F^2 + \lambda R(\sigma(\bW)) + \frac{\kappa}{2}\|\bW - \bX\|_F^2.
\end{equation}
To find $\bX(\bW)$, the minimizer of~\eqref{eq:mat_sr3} with respect to $\bX$, we solve a least squares 
problem. The $\bW$ update requires thresholding the singular values of $\bX(\bW)$. 
We compare the results from four formulations, SR3 $\ell_0$, SR3 $\ell_1$, classic $\ell_0$ and classic $\ell_1$, i.e.
the equations
\begin{equation}
\label{eq:mat_pg}
\min_{\mathbf{X}} \frac{1}{2}\|\mathcal{A}(\mathbf{X}) - \mathbf{D}\|_F^2 + \lambda R(\sigma(\mathbf{X}))
\end{equation}
and
\begin{equation}
\label{eq:mat_sr3}
\min_{\mathbf{X}, \mathbf{W}} \frac{1}{2}\|\mathcal{A}(\mathbf{X})  - \mathbf{D}\|_F^2 + \lambda R(\sigma(\mathbf{W})) + \frac{\kappa}{2}\|\mathbf{W} - \mathbf{X}\|_F^2 \; ,
\end{equation}
where $R$ can be either $\ell_1$ or $\ell_0$.
To generate figures from \SR~solutions, we look at the signal matrix $\mathbf{X}$ rather than the auxiliary matrix $\mathbf{W}$, since we 
want the interpolated result  rather a support estimate, as in the compressive sensing examples. 

In Figure~\ref{fig:mat_data}, 85\% of the data is missing. We arrange the frequency slice
into a $401 \times 401$ matrix, and then transform the data into the midpoint-offset domain following~\cite{aravkin2014fast}, 
with $m = \frac{1}{2}(s+r)$ and $h = \frac{1}{2}(s-r)$, increasing the dimension to $401 \times 801$. 
We then solve~\eqref{eq:mat_sr3} to interpolate the slice, and compare with the original to get a signal-to-noise ratio (SNR)
of $9.7$ (last panel in Fig.~\eqref{fig:mat_data}). 
The SNR obtained by solving~\eqref{eq:mat_pg} is $9.2$. 

\noindent
\blue{{\bf Parameter selection.}
We choose $\kappa = 0.5$ for all the experiments and do a cross validation for $\lambda$.
When $R = \ell_1$, we range $\lambda$ from 5 to 8 and when $R = \ell_0$, we range $\lambda$ from 200 to 400.
}

\begin{figure}[h]
\centering
\includegraphics[width=0.49\textwidth]{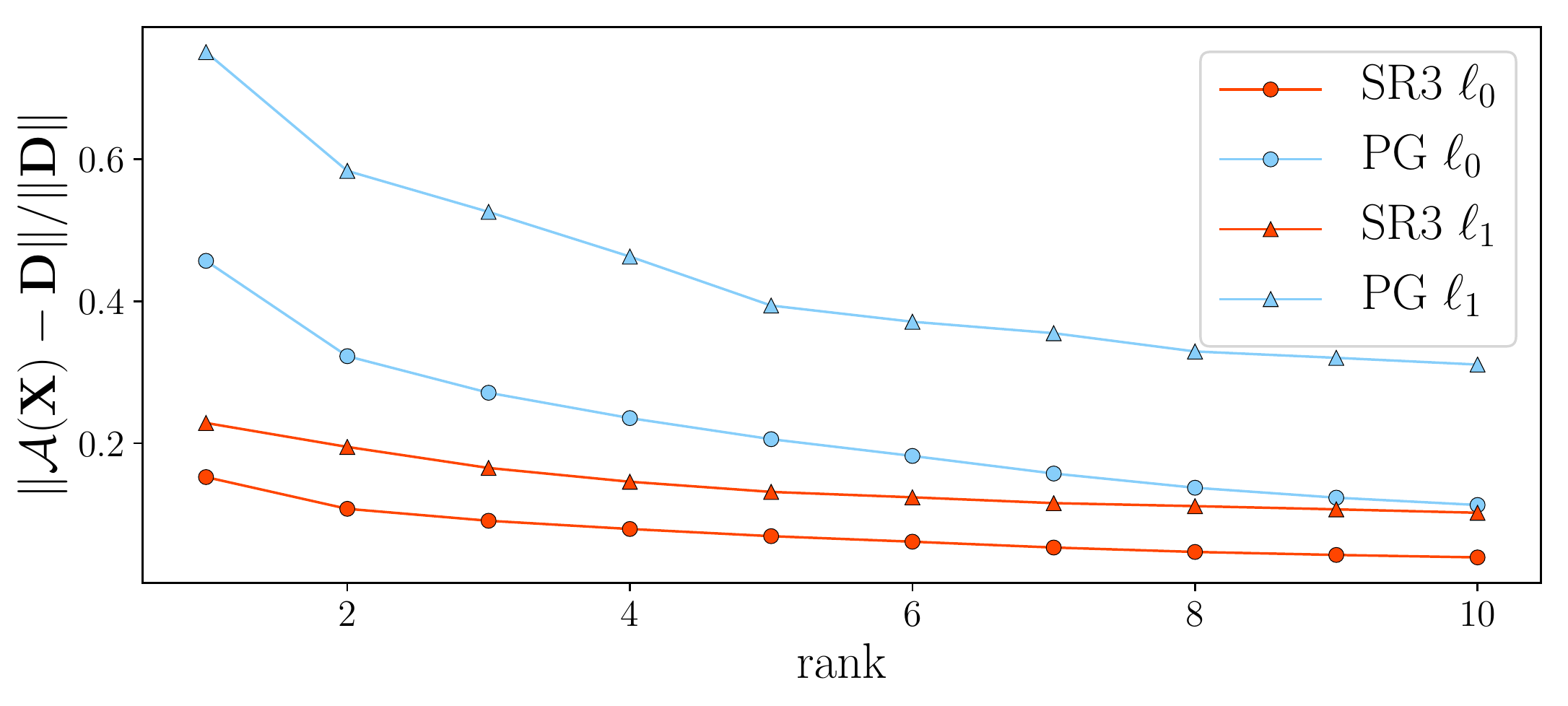}
\caption{\label{fig:pareto} Pareto frontiers (best fit achievable for each rank) for~\eqref{eq:mat_pg} with $R = \ell_1, R = \ell_0$, and for corresponding 
\SR~formulations~\eqref{eq:mat_sr3}, describing the best fits of observed values achievable for a given rank 
(obtained across regularizers for the four formulations). $\ell_0$ formulations are more efficient than 
those with $\ell_1$, and \SR~formulations~(\ref{eq:mat_sr3}) are more efficient classic formulations~(\ref{eq:mat_pg}).}
\end{figure}

\noindent
\blue{{\bf Results.}
Results are shown in Figures~\ref{fig:mat_result} and~\ref{fig:pareto}. The relative quality of the images is hard to compare with the naked eye, 
so we compute the Signal to Noise Ratio (SNR) with respect to the original (fully sampled) data to present a comparison. 
SR3 fits original data better than the solution of~\eqref{eq:mat_pg}, obtaining a maximum SNR of 12.6,  \blue{see Figure~\ref{fig:mat_result}}.}
%

We also generate Pareto curves for the four approaches, plotting achievable misfit on the observed data against the ranks of the solutions. 
Pareto curves for $\ell_0$ formulations  lie below those of $\ell_1$ formulations, i.e. using the 0-norm allows better 
data fitting for a given rank, and equivalently a lower rank at a particular error level, see~Figure~\ref{fig:pareto}. 
The Pareto curves obtained using the \SR~approach are lower still, through the relaxation. 

\subsection{\SR~for Group Sparsity}
Group sparsity is a composite sparse regularizer used in
multi-task learning to regularize under-determined
learning tasks by introducing redundancy in
the solution vectors.
Consider a set of under-determined linear systems,
\[
 \bb_i =  \bA_i \bx_i + \sigma \bm \epsilon_i, \quad i = 1,\ldots,k,
\]
where $\bA_i\in\R^{m_i\times n}$ and $m_i < n$. 
If we assume a priori that some of these systems might
share the same solution vector,
we can formulate the problem of recovering the $\bx_i$ as
\[
\min_{\bx_i}~~\frac{1}{2}\sum_{i=1}^k\| \bA_i  \bx_i -  \bb_i\|_2^2 + \lambda\sum_{i=1}^{k-1}\sum_{j=i+1}^k\| \bx_i -  \bx_j\|_2
\]
where the $\ell_2$ norm promotes
sparsity of the differences $\bx_i-\bx_j$ (or, equivalently,
encourages redundancy in the $\bx_i$).
To write the objective in a compact way, set
\[
 \bx =\begin{bmatrix}
\bx_1\\
\vdots\\
 \bx_k
\end{bmatrix}, \quad
\bb = \begin{bmatrix}
\bb_1\\
\vdots\\
 \bb_k
\end{bmatrix}, \quad
 \bA = \begin{bmatrix}
 \bA_1 & & \\
& \ddots & \\
& &  \bA_k
\end{bmatrix}.
\]
We can then re-write the optimization problem as
\[
\min_{\bx}~~\frac{1}{2}\| \bA \bx -  \bb\|_2^2 + \lambda\sum_{i=1}^{k-1}\sum_{j=i+1}^k\| \mathbf{D}_{ij}  \bx\|_2 \; ,
\]
where $ \mathbf{D}_{ij}\bx $ gives the pairwise differences
between $ \bx_i$ and $ \bx_j$. There is no simple primal
algorithm for this objective, as $\|\cdot\|_2$ is not smooth
and there is no efficient prox operation for the composition of $\|\cdot\|_2$
with the mapping $\mathbf{D}$.
\begin{figure}[t]
\centering
\includegraphics[width=0.3\textwidth]{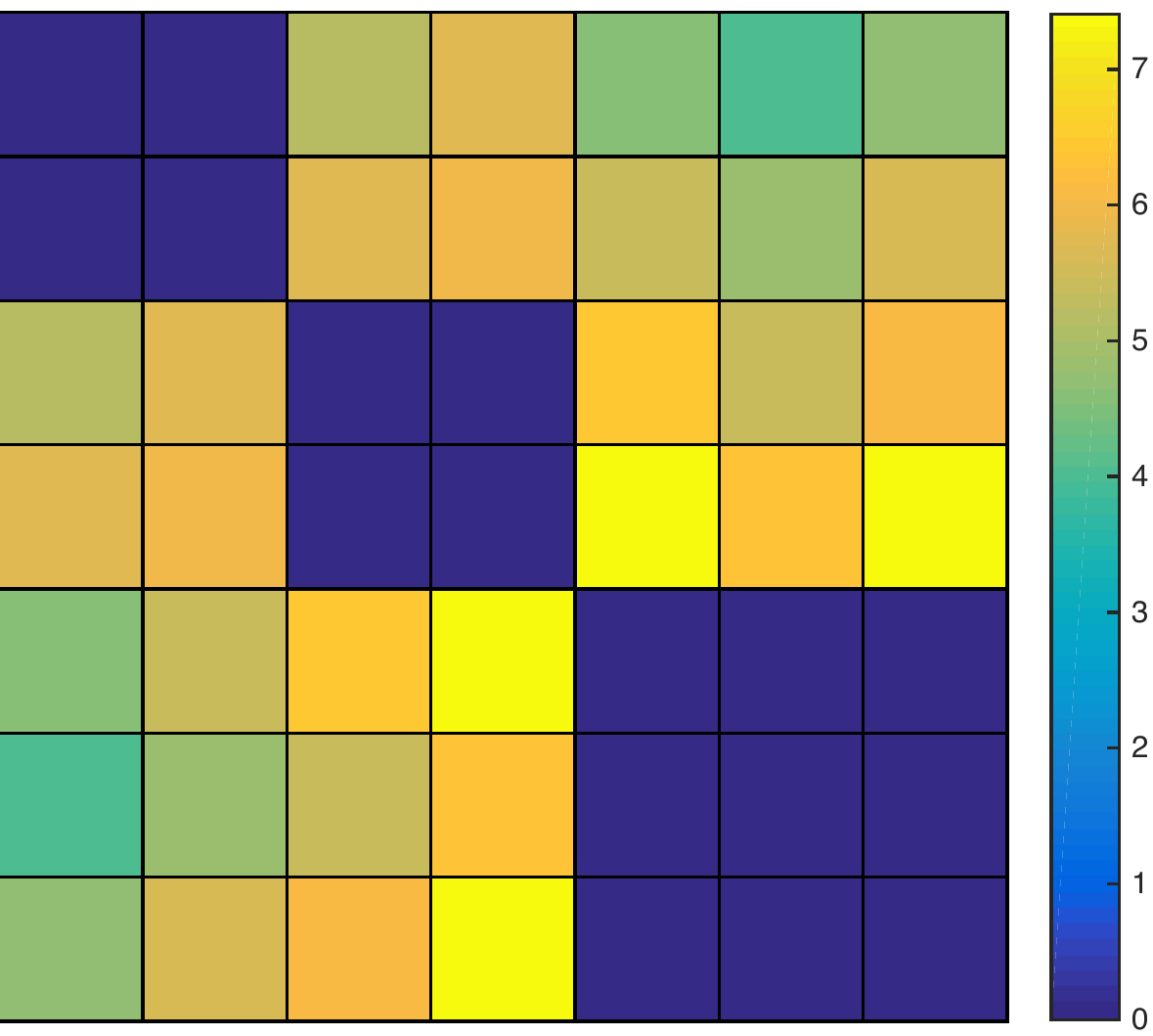}
\caption{Pairwise distance between all decision variables of different tasks obtained by \SR.}
\label{fig:multi_tasks}
\end{figure}

Applying the \SR~approach, we introduce the variables
$\bw_{ij}$ to approximate $\mathbf{D}_{ij}\bx$ and obtain
\[
\begin{aligned}
\min_{ \bx,  \bw} &~~\frac{1}{2}\| \bA  \bx -  \bb\|_2^2 + \lambda\sum_{i=1}^{k-1}\sum_{j=i+1}^k\| \bw_{ij}\|_2 \\
& \qquad+ \frac{\kappa}{2}\sum_{i=1}^{k-1}\sum_{j=i+1}^k\| \bw_{ij} -  \mathbf{D}_{ij} \bx\|_2^2 \; .
\end{aligned}
\]
\noindent
\blue{{\bf Problem setup.}
We set up a synthetic problem
with $n = 200$, $m_i = 150$, and $k = 7$.
The $ \bA_i$ are  random Gaussian matrices
and we group the true underlying signal as follows: 
\[
\bx_1 =  \bx_2, \quad  \bx_3 =  \bx_4, \quad  \bx_5 =  \bx_6 =  \bx_7 
\]
where the generators are sampled form a Gaussian distribution.
We set the noise level to $\sigma = 0.1$.}

\noindent
\blue{{\bf Parameter selection.}
We select optimization parameters to be $\lambda = 10$ and $\kappa=1$.
}

\noindent
\blue{{\bf Results.}
The pairwise distance of the result is shown in Figure~\ref{fig:multi_tasks}.
The groups have been successfully recovered. If we directly use the $\bx$
from the \SR~solution, we obtain $47\%$ relative error.
However, using the pattern discovered by $\bw$ to regroup
the least square problems, namely combine
$ \bA_1,  \bA_2$ and $ \bb_1,  \bb_2$ to solve for the first group of variables, $ \bx_1 =  \bx_2$, and so on, we improve
the result significantly to $1\%$ relative error (which is essentially
optimal given the noise).}

\section{Discussion and Outlook}
Sparsity promoting regularization of regression problems continues to play a critical role in obtaining actionable and interpretable models from data.  Further, the robustness, computational efficiency, and generalizability of such algorithms is required for them to have the potential for broad applicability across the data sciences.  The \SR~algorithm developed here satisfies all of these important criteria and provides a broadly applicable, simple architecture that is better than state-of-the-art methods for compressed sensing, matrix completion, LASSO, TV regularization, and group sparsity.  Critical to its success is the relaxation that splits {sparsity} and {accuracy} requirements. 

\blue{
The \SR~approach introduces an additional relaxation parameter. In the empirical results presented here, we did not vary $\kappa$ significantly, showing that 
for many problems, choosing $\kappa \approx 1$ can improve over the state of the art. The presence of $\kappa$ affects the regularization parameter $\lambda$, 
which must be tuned even if a good $\lambda$ is known for the original formulation. Significant improvements can be achieved by  
choices of the pair $(\kappa, \lambda)$; we recommend using cross-validation, and leave automatic strategies for parameter tuning to future work.  
}

The success of the relaxed formulation also suggests broader applicability of \SR.  
Specially, we can also consider the general optimization problem associated with nonlinear functions, such as the training of neural networks, 
optimizing over a set of supervised input-output responses that are given by a nonlinear function $f(\cdot)$ with constraints. 
 The relaxed formulation of \eqref{eq:generalxw} generalizes to
\begin{equation}
\label{eq:NNxw}
\min_{\bx,\bw} f( \bA,\bx,\bb) + \lambda R(\bw) + \frac{\kappa}{2} \|\mathbf{C}\bx-\bw\|^2.
\end{equation}
Accurate and sparse solutions for such neural network architectures can be more readily generalizable, analogous with how \SR~helps to achieve robust variable selection in sparse linear models.
The application to neural networks is beyond the scope of the current manuscript, but the architecture proposed has great potential for broader applicability.

\appendix

\section{}

We review necessary preliminaries from the optimization literature, and then 
present a series of theoretical results that explain some of the properties of SR3 solutions and characterize 
convergence of the proposed algorithms. 

\section*{Mathematical Preliminaries}
Before analyzing SR3, we give some basic results from the non-smooth optimization literature.

\subsection*{Subdifferential and Optimality}
In this paper, we work with nonsmooth functions, both convex and nonconvex. 
Given a convex nonsmooth function  $f\colon\R^n\to\overline \R$ and a point $\bar x$ with $f(\bar x)$ finite, 
the {\it subdifferential} of $f$ at $\bar x$, denoted
$\partial f(\bar x)$, is the set of all vectors $v$ satisfying
\[
f(x)\geq f(\bar x)+\langle v,x-\bar x\rangle \quad \forall \; x.
\]
The classic necessary stationarity condition $0 \in \partial f(\bar x)$ implies $f(x) \geq f(\bar x)$ for all $x$, i.e. global optimality. 
The definition of subdifferential must be amended for the general nonconvex case. 
Given an arbitrary function $f\colon\R^n\to\overline \R$ and a point $\bar x$ with $f(\bar x)$ finite,
the  {\em Fr\'{e}chet subdifferential} of $f$ at $\bar x$, denoted $\hat \partial f(\bar x)$, is the set of all vectors $v$ satisfying
\[f(x)\geq f(\bar x)+\langle v,x-\bar x\rangle+o(\|x-\bar x\|)\qquad \textrm{ as }x\to \bar x.\]
Thus the inclusion $v\in\hat\partial f(\bar x)$ holds precisely when the affine function
$x\mapsto f(\bar x)+\langle v,x-\bar x\rangle$ underestimates $f$ up to first-order near $\bar x$.
In general, the limit of Fr\'{e}chet subgradients $v_i\in \hat\partial f(x_i)$, along a
sequence $x_i\to\bar x$, may not be a Fr\'{e}chet subgradient at the limiting point $\bar x$.
Therefore, one formally enlarges the Fr\'{e}chet subdifferential and
defines the {\em limiting subdifferential} of $f$ at $\bar x$, denoted
$\partial f(\bar x)$, to consist of all vectors $v$ for which there exist
sequences $x_i$ and $v_i$, satisfying $v_i\in \partial f(x_i)$ and
$(x_i,f(x_i),v_i)\to (\bar x, f(\bar x),v)$. 
In this general setting, the condition $0 \in \partial f(\bar x)$ is necessary but not sufficient. 
However, stationary points are the best we can hope to find using iterative methods, 
and distance to stationarity serves as a way to detect convergence and analyze algorithms.  
In particular, we design and analyze algorithms that find the stationary points of
\eqref{eq:generic} and \eqref{eq:valueExp}, which are
defined below, for both convex and nonconvex regularizers $R(\cdot).$
\begin{definition}[Stationarity]
\label{df:stationary}
We call $\hat \bx$ the stationary point of \eqref{eq:generic} if,
\[
\bm 0 \in \bA^\top(\bA \hat \bx - \bb) + \lambda \bC^\top \partial R(\hat \bx).
\]
And $(\hat\bx, \hat\bw)$ the stationary point of \eqref{eq:valueExp} if,
\[\begin{aligned}
\bm 0 &= \bA^\top(\bA \hat \bx - \bb) + \kappa\bC^\top(\bC \hat \bx - \hat \bw),\\
\bm 0 &\in \lambda \partial R(\hat \bw) + \kappa(\hat \bw - \bC \hat \bx).
\end{aligned}\]
\end{definition}

\subsection*{Moreau Envolope and Prox Operators}
For any function $f$ and real $\eta>0$, the {\em Moreau envelope}
and the {\em proximal mapping} are defined by 
\begin{align}
  f_{\eta}(x)&:=\inf_{z}\, \left\{ f(z)+\tfrac{1}{2\eta}\|z-x\|^2\right\},
  \label{eq:envelope} \\
\prox_{{\eta}f}(x) &:=\argmin_{z}\, \left\{ \eta f(z)+\tfrac{1}{2}\|z-x\|^2\right\},
\end{align}
respectively.
The Moreau envelope has a smoothing effect on
convex functions, characterized by the following theorem.
Note that a proper function $f$ satisfies that $f>-\infty$ and
it takes on a value other than $+\infty$ for some $x$. A closed
function satisfies that $\{ x: f(x) \leq \alpha \}$ is a closed set
for each $\alpha \in \R$.
\begin{theorem}[Regularization properties of the envelope]\label{th:lip_cont}
  Let $f\colon\R^n\to\R$ be a proper closed convex function. Then $f_{\eta}$ is
  convex and $C^1$-smooth with
$$\nabla f_{\eta}(x)=\tfrac{1}{\eta}(x-\prox_{\eta f}(x))
\quad \textrm{ and }\quad \lip(\nabla f_{\eta})\leq \tfrac{1}{\eta}.$$
If in addition $f$ is $L$-Lipschitz, then the envelope $f_{\eta}(\cdot)$ is
$L$-Lipschitz and  satisfies
\begin{equation}\label{eqn:moreau}
	0\leq f(x)- f_{\eta}(x)\leq \frac{L^2\eta}{2}\qquad \textrm{ for all } x\in\R^n.
\end{equation}
\end{theorem}
\begin{proof}
See Theorem 2.26 of \cite{RW98}.
\end{proof}
However, when $f$ is not convex, $f_\eta$ may no longer be smooth as we show in Figure~\ref{fig:l0_envelope} where we use $\ell_0$ as an example.
\begin{figure}[t]
\centering
\includegraphics[width=0.5\textwidth]{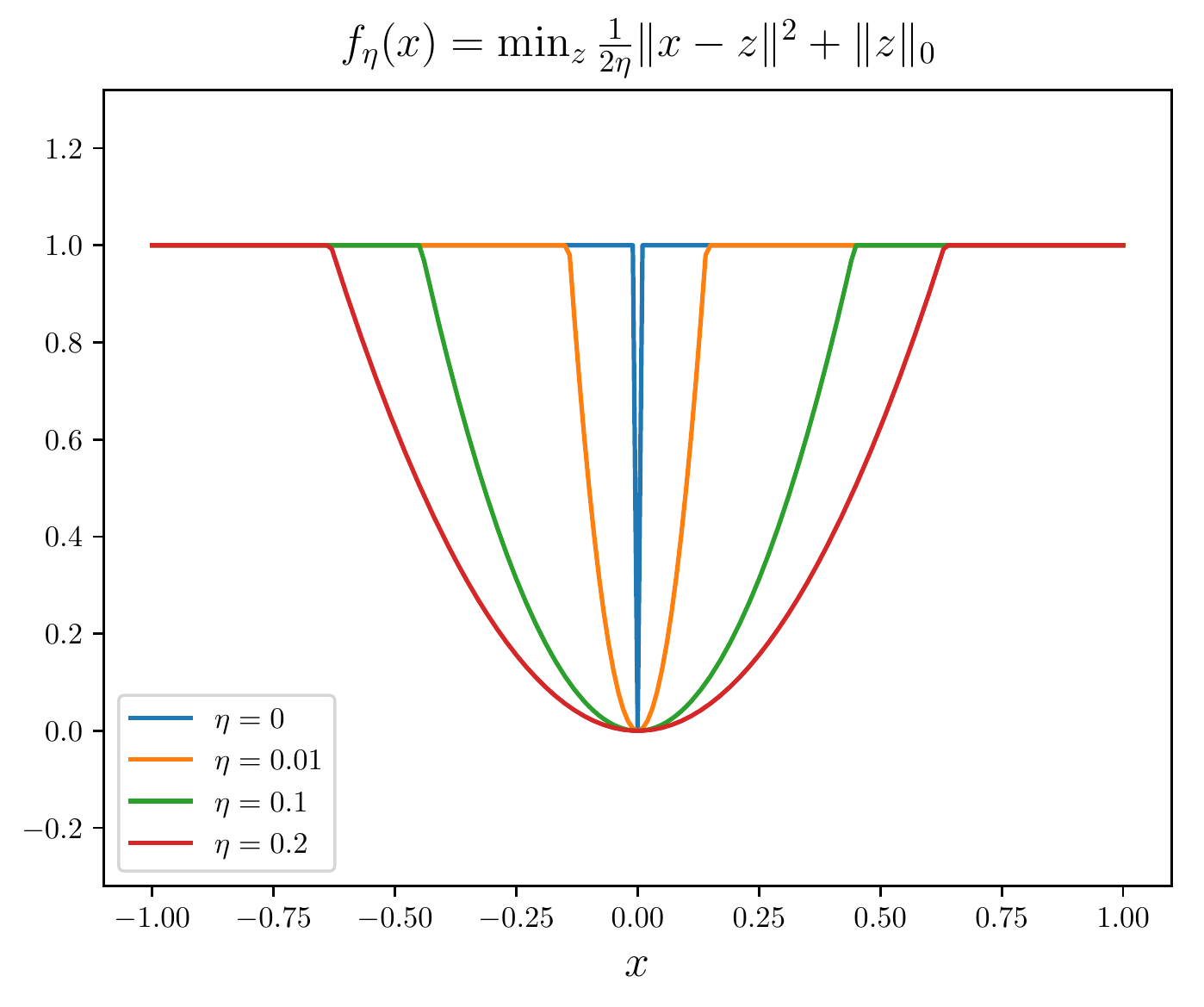}
\caption{Envelope functions indexed by the parameter $\eta$, for $f = \|\cdot\|_0$. In contrast to the convex case, here all $f_\eta$ are nonsmooth and nonconvex. }
\label{fig:l0_envelope}
\end{figure}

\subsection*{Common Prox Operators}
The prox operator is useful when designing algorithms that handle
non-smooth and non-convex functions.
Its calculation is often straightforward when the function $f$
decouples element-wise.
To illustrate the idea, we derive proximal mappings for $\ell_1, \ell_0, \ell_2^2$, and $\ell_2$. 
Many more operators can be found e.g. in~\cite{combettes2011proximal}.

\begin{itemize}
\item $f(\cdot) = \|\cdot\|_1$. The $\ell_1$ norm is a convex nonsmooth penalty often used to promote sparse solutions in 
regression problems.
  We include a derivation of the proximity operator for this problem
  and the remaining operators have similar derivations.
\begin{lemma}[$\ell_1$]
  The prox operator of $\ell_1$ is an element-wise soft-thresholding
  action on the given vector.
  \begin{equation} \label{eq:proxl1def}
\begin{aligned}
  \bx &= \prox_{\eta f}(\bm y) = \argmin_{\bx}~\frac{1}{2}\|\bx - \bm y\|^2 + \eta\|\bx\|_1~~\Rightarrow~~\\
  x_i & = \begin{cases}
    y_i - \eta, & y_i > \eta\\
    0, & |y_i| \le \eta\\
    y_i + \eta, & y_i < -\eta
  \end{cases}.
  \end{aligned}
  \end{equation}
\end{lemma}

\begin{proof}
  Note that the optimization problem may be written as
  \begin{equation}
  \begin{aligned}
    \argmin_{\bx}~&\frac{1}{2}\|\bx - \bm y\|^2 + \eta\|\bx\|_1 \\
    &= \argmin_{\bx}~
    \frac{1}{2}\sum_{i=1}^n (x_i-y_i)^2 + \eta |x_i| \; ,
\end{aligned}
  \end{equation}
  i.e. the problem decouples over the elements of $\bm y$.
  For each $i$, the optimization problem has the subdifferential

  \begin{equation}
  \begin{aligned}
  &  \partial_{x_i} \left ( \frac{1}{2} (x_i-y_i)^2 + \eta |x_i| \right )\\
&    = \begin{cases}
    x_i - y_i + \eta, & x_i > 0\\
    x_i - y_i + \{z : |z| \le \eta\} , & x_i = 0\\
    x_i -y_i - \eta, & x_i < 0
  \end{cases} \; .
  \end{aligned}
  \end{equation}
  After checking the possible stationary points given this formula
  for the subdifferential, it is simple to derive \eqref{eq:proxl1def}.
\end{proof}

\item $f(\cdot) = \|\cdot\|_0$. 
  The $\ell_0$ penalty directly controls the number of
  non-zeros in the vector instead of penalizing the
  magnitude of elements as $\ell_1$ does.
  However, it is non-convex and in practice regression formulations with $\ell_0$ 
  regularization can be trapped in local minima instead of finding the true support.
  \begin{lemma}[$\ell_0$]
  The prox operator of $\ell_0$ is simple, element-wise
  hard-thresholding:
  \begin{equation}
  \label{eq:l0}
\begin{aligned}
\bx &= \prox_{\eta f}(\bm y) = \argmin_{\bx}~\frac{1}{2}\|\bx - \bm y\|^2 + \eta\|\bx\|_0~~\Rightarrow~~\\
x_i &= \begin{cases}
y_i, & |y_i| > \sqrt{2\eta}\\
0, & |y_i| \le \sqrt{2\eta}
\end{cases}.
\end{aligned}
\end{equation}
  \end{lemma}

\begin{proof}
Analogous to the $\ell_1$, the prox problem for $\ell_0$ can be decoupled across coordinates: 
\[
\frac{1}{2}\|\bx - \bm y\|^2 + \eta\|\bx\|_0 = \argmin_{\bx}~
    \frac{1}{2}\sum_{i=1}^n (x_i-y_i)^2 + \eta \mathrm{1}_{\{x_i = 0\}} \; .
\]
From this formula, it is clear that the only possible solutions for each coordinate are $x_i = 0$ or $x_i = y_i$. 
The formula~\eqref{eq:l0} follows from checking the conditions for these cases. 
\end{proof}

\item $f(\cdot) = \frac{1}{2}\|\cdot\|^2$.
  The $\ell_2^2$ penalty can be used as a smooth and convex
  penalty which biases towards zero.
  When combined with linear regression, it is commonly known as
  ridge regression.
  \begin{lemma}[$\ell_2^2$]
  The prox of $\ell_2^2$ is scaling.
\[
\bx = \prox_{\eta f}(\bm y) = \argmin_{\bx}~\frac{1}{2}\|\bx - \bm y\|^2 + \frac{\eta}{2}\|\bx\|^2 = \frac{1}{1+\eta}\bm y.
\]
  \end{lemma}
\begin{proof}
The proof follows directly from calculus. 
\end{proof}

\item $f(\cdot) = \|\cdot\|$.
  The $\ell_2$ norm adds a group sparsity prior, i.e. the vector $\bx$ is
  biased toward being the zero vector.
  Often, this penalty is applied to each column of a matrix of variables.
  Unlike the prox operators above, $\|\cdot\|$ (by design) does not
  decouple into scalar problems.
  Fortunately, a closed form solution is easy to obtain.
  \begin{lemma}
\[
\begin{aligned}
\bx &= \prox_{\eta f}(\bm y) = \argmin_{\bx}~\frac{1}{2}\|\bx - \bm y\|^2 + \eta\|\bx\|~~\Rightarrow~~\\
\bx &= \begin{cases}
\frac{\|\bm y\| - \eta}{\|\bm y\|}\bm y, & \|\bm y\| > \eta\\
\bm 0, & \|\bm y\| \le \eta
\end{cases}.
\end{aligned}
\]
  \end{lemma}
  \begin{proof}
    Observe that for any fixed value of $\|\bx\|$ the objective
    \begin{equation}
      \frac{1}{2}\|\bx - \bm y\|^2 + \eta\|\bx\|
    \end{equation}
    is minimized by taking $\bx$ in the direction of $\bm y$.
    This reduces the problem to finding the optimal value of
    $\|\bx\|$, for which the same reasoning as the $\ell_1$
    penalty applies.
  \end{proof}

\end{itemize}

\subsection*{Proximal Gradient Descent}

\begin{algorithm}[h!]
\caption{Proximal gradient descent}
\label{alg:pg_proto}
\begin{algorithmic}[1]
\State {\bfseries Input:} $\bx_0,\eta$
\State {\bfseries Initialize:} $k=0$
\While{not converged}
\Let{$k$}{$k+1$}
\Let{$\bx_k$}{$\prox_{\eta g}(\bx_{k-1} - \eta \nabla f(\bx_{k-1}))$}
\EndWhile
\State {\bfseries Output:} $\bx_k$
\end{algorithmic}
\end{algorithm}

Consider an objective of the form $p(x) = f(x) + g(x)$.
Given a step size $t$, the proximal gradient
descent algorithm is as defined in Algorithm~\ref{alg:pg}
\cite{combettes2011proximal}.
This algorithm has been studied extensively.
Among other results, we have
\begin{theorem}[Proximal Gradient Descent] \label{thm:pg}
Assume $p = f+g$ and both $p$ and $g$ are closed
convex functions. Let $p^*$ denote the optimal function
value and $\bx^*$ denote the optimal solution.
\begin{itemize}
\item If $\nabla f$
  is $\beta$ Lipschitz continuous, then, setting the
  step size as $1/\beta$, the iterates generated
  by proximal gradient descent satisfy
\[p(\bx^k) - p^* \le \frac{\beta\|\bx^0 - \bx^*\|^2}{2(k+1)}.\]
\item Furthermore, if $p$ is also $\alpha$ strongly
  convex, we have,
\[
\|\bx^k - \bx^*\|^2 \le \left(1 - \frac{\alpha}{\beta}\right)^k\|\bx^0 - \bx^*\|^2.
\]
\end{itemize}
\end{theorem}

These results are well known; see e.g.~\cite{beck2009fastlin,combettes2011proximal,parikh2014proximal}
and the tutorial section 4.4 of~\cite{aravkin2017generalized}.

\section*{Theoretical Results}
In the main text, it is demonstrated that SR3  \eqref{eq:valueExp}
outperforms the standard regression problem \eqref{eq:generic},
achieving faster convergence and obtaining higher quality solutions.
%
%
%
Here, we develop some theory to explain the performance of SR3
from the perspective of the relaxed coordinates, $\bw$.
We obtain an explicit formula for the SR3 problem in $\bw$ alone
and then analyze the spectral properties of that new problem,
demonstrating that the conditioning of the $\bw$ problem is
greatly improved over that of the original problem.
We also obtain a quantitative measure of the distance between
the solutions of the original problem and the SR3 relaxation.
%

%

%
\subsection*{Spectral Properties of $\bF_\kappa$}
%

\subsubsection{Proof of Theorem~\ref{thm:sol}}

  The first property can be verified by direct calculation. We have
  \begin{align*}
    \bF_{\kappa}^\top \bF_\kappa \bw - \bF_\kappa^\top \bg_\kappa =& (\kappa \mbf I - \kappa^2 \bC \bH_\kappa^{-1} \bC^\top)\bw - \kappa \bC \bH_\kappa^{-1} \bA^\top \bb\\
    = &\kappa \bH_\kappa^{-1}[(\bH_\kappa - \kappa \mbf I)\bw - \bA^\top \bb] \\
     = &\kappa \bH_\kappa^{-1}(\bA^\top \bA \bw - \bA^\top \bb)
  \end{align*}
  so that $\bF_\kappa^\top \bF_\kappa \bw - \bF_\kappa^\top \bg_\kappa = \bm 0
  \iff \bA^\top \bA \bw + \mbf A^\top \bb = \bm 0$.
\blue{  
By simple algebra, we have,
\begin{equation}
\label{eq:ftf}
\begin{aligned}
\bF_\kappa^\top \bF_\kappa &= \kappa \mbf I - \kappa^2 \bC \bH_\kappa^{-1} \bC^\top \\
\sigma_i(\bF_\kappa^\top \bF_\kappa) &= \kappa - \kappa^2\sigma_{n-i+1}(\bC \bH_\kappa^{-1} \bC^\top).
\end{aligned}
\end{equation}
Since $\bC \bH_\kappa^{-1} \bC^\top$ and $\bF_\kappa^\top \bF_\kappa$ are positive semi-definite matrices,
we have $\mbf 0 \preceq \bF_\kappa^\top \bF_\kappa \preceq \kappa \mbf I$. Denote the SVD for $\bC$ by
\(
\bC = \mbf U_c \mbf \Sigma_c \mbf V_c^\top.
\)
When $n \ge d$ and $\bC$ is full rank, we know $\mbf \Sigma_c$ is invertible and $\mbf V_c$ is orthogonal.
Then
\begin{align*}
\bC \bH_\kappa^{-1} \bC^\top &= \mbf U_c \mbf \Sigma_c \mbf V_c^\top (\mbf A^\top \mbf A + \kappa \mbf V_c \mbf \Sigma_c^2 \mbf V_c^\top)^{-1} \mbf V_c \mbf \Sigma_c \mbf U_c^\top\\
&=\mbf U_c \mbf  (\mbf \Sigma_c^{-1} \mbf V_c^\top \mbf A^\top \mbf A  \mbf V_c \mbf \Sigma_c^{-1}+ \kappa \mbf I)^{-1} \mbf U_c^\top
\end{align*}
This gives a lower bound of the spectrum of $\bC \bH_\kappa^{-1} \bC^\top$,
\begin{align*}
&\smin(\mbf \Sigma_c^{-1} \mbf V_c^\top \mbf A^\top \mbf A  \mbf V_c \mbf \Sigma_c^{-1}) \ge \smin(\mbf A^\top \mbf A) /\smax(\bC^\top \bC)\\
\Rightarrow~~
&\smax(\bC \bH_\kappa^{-1} \bC^\top) \le 1/(\smin(\mbf A^\top \mbf A) /\smax(\bC^\top \bC) + \kappa)
\end{align*}
Then we obtain the conclusion,
\[
\begin{aligned}
&\smin(\bF_\kappa^\top \bF_\kappa) \ge \kappa - \frac{\kappa^2}{\smin(\mbf A^\top \mbf A)/\smax(\bC^\top \bC) + \kappa}\\
 &= \frac{\smin(\mbf A^\top \mbf A)/\smax(\bC^\top \bC)}{1 + \smin(\mbf A^\top \mbf A)/(\kappa\smax(\bC^\top \bC))}.
\end{aligned}
\]
}

When $\bC = \bI$, 
  we have that 
  \begin{align*}
    \bF_\kappa^\top \bF_\kappa &= \kappa [\mbf I - \kappa (\mbf A^\top \mbf A + \kappa \mbf I)^{-1}]\\
    &= \bA^\top(\mbf I + \bA \bA^\top/\kappa)^{-1}\bA
  \end{align*}
  Assume $\bA\in\R^{m \times n}$ has the singular value decomposition (SVD)
  $\bA = \mbf U \mbf \Sigma \mbf V^\top$, where $\mbf U \in \R^{m \times m}$,
  $\mbf \Sigma \in \R^{m \times m}$, and $\mbf V \in \R^{m \times m}$. We have
  \[
  \bF_\kappa^\top \bF_\kappa = \mbf V \mbf \Sigma^\top(\mbf I
  + \mbf \Sigma \mbf \Sigma^\top/\kappa)^{-1}\mbf \Sigma \mbf V^\top.
  \]
  Let $\mathbf{\hat{\Sigma}}\in\R^{l \times l}$ denote the reduced diagonal part
  of $\mbf \Sigma$, i.e. the top-left $l\times l$ submatrix of $\mbf \Sigma$
  with $l = \min(m,n)$.
  When $m \ge n$, we have
\begin{equation} \label{eq:singmgen}
\mbf \Sigma = \begin{bmatrix}
\mathbf{\hat{\Sigma}} \\
\mbf 0
\end{bmatrix}, \quad
\bF_\kappa^\top \bF_\kappa = \mbf V \mathbf{\hat{\Sigma}}^\top(\mbf I + \mathbf{\hat{\Sigma}}^2/\kappa)^{-1} \mathbf{\hat{\Sigma}} \mbf V^\top
\end{equation}
And when $m < n$,
\begin{equation} \label{eq:singmltn}
\mbf \Sigma = \begin{bmatrix}
\mathbf{\hat{\Sigma}} & \mbf 0
\end{bmatrix}, \quad
\bF_\kappa^\top \bF_\kappa = \mbf V\begin{bmatrix}
\mathbf{\hat{\Sigma}}^\top(\mbf I + \mathbf{\hat{\Sigma}}^2/\kappa)^{-1} \mathbf{\hat{\Sigma}} & \mbf 0\\
\mbf 0 & \mbf 0
\end{bmatrix} \mbf V^\top
\end{equation}
\eqref{eq:FtFCI} and~\eqref{eq:svFtFCI} follow immediately.

%
%

  Note that the function
  \[
  \frac{x}{\sqrt{1+x^2/a}}
  \]
  is an increasing function of $x$ when $x, a > 0$. Therefore, by \eqref{eq:svFtFCI}, we have
  \[
  \begin{aligned}
  \sigma_{\textrm{max}}(\bF_\kappa) &= \frac{\sigma_{\textrm{max}}(\bA)}
        {\sqrt{1+\sigma_{\textrm{max}}(\bA)^2/\kappa}} \quad \mbox{ and } \\
        \sigma_{\textrm{min}}(\bF_\kappa) &= \frac{\sigma_{\textrm{min}}(\bA)}
              {\sqrt{1+\sigma_{\textrm{min}}(\bA)^2/\kappa}} \; .
              \end{aligned}
              \]
              \eqref{eq:condf} follows by the definition of the
              condition number.

\subsubsection{Proof of Theorem~\ref{thm:general_lr}.}
%
For the iterates of the proximal gradient method, we have
\[
\bx_{k+1} = \argmin_{\bx} \frac{1}{2}\|\bx - (\bx_k - \eta \nabla f(\bx_k))\|^2 + \eta g(\bx)
\]
and from the first order optimality condition we have
\[\begin{aligned}
\bm 0 &\in \bx_{k+1} - \bx_k + \eta \nabla f(\bx_k) + \eta\partial g(\bx_{k+1})\\
\Rightarrow~~&\frac{1}{\eta}(\bx_k - \bx_{k+1}) + \nabla f(\bx_{k+1}) - \nabla f(\bx_k) \\
& \quad \in \nabla f(\bx_{k+1}) + \partial g(\bx_{k+1})\\
\Rightarrow~~&(\|\bA\|_2^2 \mathbf I - \bA^\top\bA)(\bx_k - \bx_{k+1}) \in \partial p(\bx_{k+1}) \; ,
\end{aligned}\]
which establishes the first statement.
Next, consider the following inequality
\begin{align*}
p(\bx_{k+1}) &= \frac{1}{2}\|\bA \bx_{k+1} - \bb\|^2 + \lambda R(\bx_{k+1})\\
&= \frac{1}{2}\|\bA \bx_{k} - \bb + \bA(\bx_{k+1} - \bx_k)\|^2 + \lambda R(\bx_{k+1})\\
&= \frac{1}{2}\|\bA \bx_k - \bb\|^2 + \lambda R(\bx_{k+1}) \\
& \qquad + \ip{\bA^\top (\bA \bx_k - \bb), \bx_{k+1} - \bx_k} \\
& \qquad + \frac{1}{2}\|\bA(\bx_{k+1} - \bx_k)\|^2\\
&\le \frac{1}{2}\|\bA \bx_k - \bb\|^2 + \lambda R(\bx_k)
-\frac{\|\bA\|_2^2}{2}\|\bx_{k+1} - \bx_k\|^2 \\
&\qquad + \frac{1}{2}\|\bA(\bx_{k+1} - \bx_k)\|^2
\; ,
\end{align*}
which implies the inequality
\[
\begin{aligned}
\ip{\bx_k - \bx_{k+1}, (\|\bA\|_2^2 \mathbf I - \bA^\top\bA)(\bx_k - \bx_{k+1})}  \\
\le p(\bx_k) - p(\bx_{k+1}) \\
\Rightarrow~~\|\bA\|_2^2 \|\bx_{k+1} - \bx_k\|^2 \le p(\bx_k) - p(\bx_{k+1}).
\end{aligned}
\]
Setting $\bv_{k+1} = (\|\bA\|_2^2 \mathbf I - \bA^\top\bA)(\bx_k - \bx_{k+1})$, we have
\[
\|\bv_{k+1}\|^2 \le \|\bA\|_2^4\|\bx_{k+1} - \bx_k\|^2 \le \|\bA\|_2^2(p(\bx_{k}) - p(\bx_{k+1})) \; .
\]
After we add up and simplify, we obtain
\[
\begin{aligned}
\frac{1}{N}\sum_{k=0}^{N-1}\|\bv_{k+1}\|^2 \le \frac{\|\bA\|_2^2}{N}(p(\bx_0) - p(\bx_N))\\
 \le \frac{\|\bA\|_2^2}{N}(p(\bx_0) - p^*) \; ,
\end{aligned}
\]
which is the desired convergence result.

\subsubsection{Proof of Theorem~\ref{col:cvx}.}

The result is immediate from combining Theorem~\ref{thm:general_lr} and Theorem~\ref{thm:sol}.

\subsubsection{Proof of Corollary~\ref{cor:ctcrates}.}
\blue{
The result is immediate from combining Theorem~\ref{thm:general_lr} and Corollary~\ref{cor:tight_frame}.
}

%

\subsection*{Characterizing Optimal Solutions of \SR}
In this section, we quantify the relation between the solution of \eqref{eq:generic} and \eqref{eq:valueExp} when $\bC = \mbf I$.
In this analysis, we fix $\kappa$ as a constant and set $\bC = \mbf I$.
\begin{lemma}[Optimality conditions for \eqref{eq:generic} and \eqref{eq:valueExp}]
\label{lm:stationarity}
Define the sets
\[\begin{aligned}
\mathcal{S}_1(\bx, \lambda_1) &= \{\bA^\top\bA \bx - \bA^\top\bb + \lambda_1 \bm v_1 : \bm v_1 \in \partial R(\bx)\}\\
\mathcal{S}_2(\bw, \lambda_2) &= \{\kappa\bH_\kappa^{-1}(\bA^\top\bA \bw - \bA^\top\bb) + \lambda_2 \bm v_2 : \bm v_2 \in \partial R(\bw)\} \; ,
\end{aligned}\]
where $\bH_\kappa = \bA^\top\bA + \kappa \bI$, as above. These sets
contain the subgradients of \eqref{eq:generic} and \eqref{eq:valueExp}.
If we assume $\hat \bx$ and $\hat \bw$ are the (stationary)
solutions of \eqref{eq:generic} and \eqref{eq:valueExp}, namely
\[
\bm 0 \in \mathcal{S}_1(\hat\bx, \lambda_1), \quad \bm 0 \in \mathcal{S}_2(\hat\bw, \lambda_2) \; ,
\]
then 
\begin{align*}
[\mbf I - (\lambda_1/\lambda_2)\kappa \bH_\kappa^{-1}](\bA^\top\bA \hat \bw - \bA^\top\bb) &\in \mathcal{S}_1(\hat\bw, \lambda_1),\\
[\kappa \bH_\kappa^{-1} - (\lambda_2/\lambda_1)\mbf I](\bA^\top\bA \hat \bx - \bA^\top\bb) &\in \mathcal{S}_2(\hat\bx, \lambda_2).
\end{align*}
\end{lemma}
\begin{proof}
As $\hat \bx$ and $\hat \bw$ are the (stationary) solutions of \eqref{eq:generic} and \eqref{eq:valueExp}, we have
\begin{align*}
\exists \bm v_1 \in \partial R(\hat \bx), \quad &\lambda_1 \bm v_1 = -(\bA^\top\bA \hat \bx - \bA^\top\bb),\\
\exists \bm v_2 \in \partial R(\hat \bw), \quad &\lambda_2 \bm v_2 = -\kappa\bH_\kappa^{-1}(\bA^\top\bA \hat \bw - \bA^\top\bb).
\end{align*}
Then, 
\begin{align*}
&\bA^\top \bA \hat \bw - \bA^\top \bb + \lambda_1 \bv_2 \in \mathcal{S}_1(\hat \bw, \lambda_1) \\
&~~\Rightarrow~~[\mbf I - (\lambda_1/\lambda_2)\kappa \bH_\kappa^{-1}](\bA^\top\bA \hat \bw - \bA^\top\bb) \in \mathcal{S}_1(\hat\bw, \lambda_1),\\
&\kappa\bH_\kappa^{-1}(\bA^\top\bA \hat \bx - \bA^\top\bb) + \lambda_2 \bm v_1 \in \mathcal{S}_2(\hat \bx, \lambda_2)\\
&~~\Rightarrow~~[\kappa \bH_\kappa^{-1} - (\lambda_2/\lambda_1)\mbf I](\bA^\top\bA \hat \bx - \bA^\top\bb) \in \mathcal{S}_2(\hat\bx, \lambda_2).
\end{align*}
\end{proof}
\subsubsection{Proof of Theorem~\ref{thm:ratio}}
Using the definitions of Lemma~\ref{lm:stationarity}, we have
\begin{align*}
  &\mathrm{dist}(\bm 0, \mathcal{S}_1(\hat\bw, \lambda_1)) \\
  &\quad \le \frac{1}{\hat \tau}\|(\hat \tau \mbf I - \kappa \bH_\kappa^{-1})(\bA^\top \bA \hat \bw - \bA^\top \bb)\|\\
&\quad  = \frac{1}{\hat \tau}\|\hat \tau \mbf I - \kappa \bH_\kappa^{-1}\|_2 \|\bA^\top \bA \hat \bw - \bA^\top \bb\|\\
&\quad  = \frac{1}{\hat \tau}\|\hat\tau \bm{1} - \kappa \sigma(\bH_\kappa^{-1})\|_\infty \|\bA^\top \bA \hat \bw - \bA^\top \bb\| \\
&\quad  = \frac{\sigma_\mathrm{max}(\bH_\kappa) - \sigma_\mathrm{min}(\bH_\kappa)}{\sigma_\mathrm{max}(\bH_\kappa) + \sigma_\mathrm{min}(\bH_\kappa)}\|\bA^\top \bA \hat \bw - \bA^\top \bb\| \; \\
&\quad  = \frac{\sigma_\mathrm{max}(\bA)^2 - \sigma_\mathrm{min}(\bA)^2}{\sigma_\mathrm{max}(\bA)^2 + \sigma_\mathrm{min}(\bA)^2 + 2\kappa}\|\bA^\top \bA \hat \bw - \bA^\top \bb\| \; .
\end{align*}

%
If $\hat \bx = \hat \bw$, then
$\bm r = \bA^\top\bA \hat \bw - \bA^\top\bb = \bA^\top\bA \hat \bx - \bA^\top\bb$ is
in the null space of $\tau \mbf I - \kappa \bH_\kappa^{-1}$, where
$\tau = \lambda_2/\lambda_1$.
This establishes a connection between $\lambda_1$ and $\lambda_2$.
For instance, we have the following result.
In the case that $\bA$ has orthogonal rows or columns, theorem~\ref{thm:ratio}
provides some explicit bounds on the distance between these solutions.
\begin{corollary}
  If $\bA^\top \bA = \mbf I$, then
  $\mathrm{dist}(\bm 0, \mathcal{S}_1(\hat\bw, \lambda_1)) = 0$, i.e.
  $\hat\bw$ is the stationary point of \eqref{eq:generic}.
  If $\bA \bA^\top = \mbf I$, then $\mathrm{dist}(\bm 0, \mathcal{S}_1(\hat\bw, \lambda_1)) \le 1/(1+2\kappa)$.
\end{corollary}
\begin{proof}
  The formula for $\bH_\kappa$ simplifies under these assumptions.
  When $\bA^\top \bA = \mbf I$, we have $\bH_\kappa = (1+ \kappa)\mbf I$ and $\sigma_\mathrm{max}(\bH_\kappa) = \sigma_\mathrm{min}(\bH_\kappa) = 1+\kappa$.
  When $\bA \bA^\top = \mbf I$, we have $\sigma_\mathrm{max}(\bH_\kappa) = 1 + \kappa$ and $\sigma_\mathrm{min}(\bH_\kappa) = \kappa$.
  Theorem~\ref{thm:ratio} then implies the result.
\end{proof}

\subsection{Implementation of $\ell_q$ proximal operator.}
\blue{
Here we summarize our implementation. 
%
%
%
The first and second derivatives are given by 
\begin{equation}
\label{eq:scalar_derivatives}
\begin{aligned}
f_{\alpha,p}'(x;z) &= \frac{1}{\alpha}(x - |z|) + p x^{p-1},\\
f_{\alpha,p}''(x;z) &= \frac{1}{\alpha} + p(p-1) x^{p-2}.
\end{aligned}
\end{equation}
The point $\tilde x = \sqrt[p-2]{-1/(\alpha p (p-1))}$ is the only inflection point of $f_{\alpha, p}$,
with $f_{\alpha,p}''(x) < 0$ for $0 \le x < \tilde x$, and
$f_{\alpha,p}''(x;z) > 0$ when $x > \tilde x$. 
\begin{itemize}
\item If $f'_{\alpha,p}(\tilde x;z) \ge 0$, we have $f'_{\alpha,p}(x;z) \ge 0$, for all $x \ge 0$.
Then $\argmin_{x \ge 0}~f_{\alpha,p}(x;z) = 0$.
\item If $f_{\alpha,p}'(\tilde x;z) < 0$, one local min $\bar x \in (\tilde x, |z|)$ exists,
and we can use Newton's method to find it. 
Then we compare the values at $0$ and $\bar x$, obtaining 
\[
\argmin_{x \ge 0}~f_{\alpha,p}(x;z) = \begin{cases}
0, & f_{\alpha,p}(0;z) \le f_{\alpha,p}(\bar x;z)\\
\bar x, & f_{\alpha,p}(0;z) > f_{\alpha,p}(\bar x;z)
\end{cases}.
\]
\end{itemize}
}

\bibliographystyle{abbrv}
\bibliography{refs,relax_si}

\end{document}